\newcommand{\longversion}[1]{#1}
\newcommand{\shortversion}[1]{}

\relax
\documentclass[letterpaper]{article} %DO NOT CHANGE THIS
\shortversion{\usepackage{aaai19}}  %Required
\longversion{
  	\usepackage[hscale=0.7,scale=0.75]{geometry}
	\usepackage[numbers,sort&compress]{natbib}
}

\usepackage{times}  %Required
\usepackage{helvet}  %Required
\usepackage{courier}  %Required
\usepackage{url}  %Required
\usepackage{graphicx}  %Required
\shortversion{
\frenchspacing  %Required
\setlength{\pdfpagewidth}{8.5in}  %Required
\setlength{\pdfpageheight}{11in}  %Required
}
\usepackage[]{hyperref}
\usepackage[ruled,vlined,linesnumbered]{algorithm2e}
\usepackage[table,x11names]{xcolor}
\usepackage{caption}
\renewcommand*{\algorithmcfname}{Listing}
\SetKw{KwAnd}{and}
\SetKwInput{KwData}{In}
\SetKwInput{KwResult}{Out}
\SetAlFnt{\small}
\SetAlCapFnt{\small}
\SetAlCapNameFnt{\small}
\SetAlCapHSkip{0pt}
\SetEndCharOfAlgoLine{}
\newcommand{\tuplecolor}[1]{\textcolor{#1}}
\newcommand{\inputPredColor}{orange!55!red}
\newcommand{\outputPredColor}{blue!45!black}
\newcommand{\statePredColor}{green!62!black}
\newcommand{\specialPredColor}{red!62!black}
\makeatletter
\newcommand{\algorithmfootnote}[2][\footnotesize]{
  \let\old@algocf@finish\@algocf@finish
  \def\@algocf@finish{\old@algocf@finish
    \leavevmode\rlap{\begin{minipage}{\linewidth}
    #1#2
    \end{minipage}}
  }
}
\makeatother

\usepackage[USenglish]{babel}
\usepackage[utf8]{inputenc}
\usepackage{microtype}

\usepackage{tikz}
\usetikzlibrary{arrows,arrows.meta,automata,shapes,calc,shadows,positioning,fit}

\usepackage{amssymb,amsfonts,amsmath,amsthm}%

\newtheorem{lemma}{Lemma}
\newtheorem{example}[lemma]{Example}

\newtheorem{proposition}[lemma]{Proposition}
\newtheorem{hypothesis}[lemma]{Hypothesis}

\newtheorem{theorem}[lemma]{Theorem}

\newtheorem{corollary}[lemma]{Corollary}

{%
  \addtocounter{lemma}{-1}
  \endgroup
}%

{%
  \addtocounter{lemma}{-1}
  \endgroup
}%

{%
  \addtocounter{lemma}{-1}
  \endgroup
}%

\newenvironment{restateproposition}[1][\unskip]{%
  \begingroup
  \def\thelemma{\ref{#1}}
}%
{%
  \addtocounter{lemma}{-1}
  \endgroup
}%

{%
  \addtocounter{lemma}{-1}
  \endgroup
}%

\usepackage[anythingbreaks]{breakurl}

\newcommand{\TODO}[1]{{\color{red}/*#1*/}}
\newcommand{\johannes}[1]{\TODO{J: #1}}

\newcommand{\problemFont}[1]{\mathsf{#1}}
\newcommand{\complexityClassFont}[1]{\ensuremath{\mathrm{#1}}}

\newcommand{\numberCred}{\problemFont{\#Cred}}

\newcommand{\Cred}{\problemFont{Cred}}
\newcommand{\Skep}{\problemFont{Skep}}
\newcommand{\SAT}{\protect\ensuremath{\problemFont{SAT}}\xspace}
\newcommand{\PCC}{\protect\ensuremath{\problemFont{\#PCred}}\xspace}
\newcommand{\countCirc}{\problemFont{\#Circumscription}}
\renewcommand{\phi}{\varphi}

\shortversion{
	\newcommand{\citex}[1]{\citeauthor{#1}~\shortcite{#1}}
	\newcommand{\citey}[1]{\citeauthor{#1},~\citeyear{#1}}
}
\longversion{
	\newcommand{\citex}[1]{\citet{#1}}
	\newcommand{\citey}[1]{\citex{#1}}
}

\newcommand{\bigO}[1]{\ensuremath{{\mathcal O}(#1)}}
\newcommand{\Nat}{\mathbb{N}} %
\newcommand{\Card}[1]{\left|#1\right|}
\newcommand{\CCard}[1]{\|#1\|}
\newcommand{\SB}{\{\,}%
\newcommand{\SM}{\mid}%
\newcommand{\SE}{\,\}}%
\def\hy{\hbox{-}\nobreak\hskip0pt}

\newcommand{\eqdef}{\ensuremath{\,\mathrel{\mathop:}=}}

\newcommand{\CONFL}{\ensuremath{{\algo{CONF}}}\xspace}
\newcommand{\ADM}{\ensuremath{{\algo{ADM}}}\xspace}
\newcommand{\PROJ}{\ensuremath{{\algo{PROJ}}}\xspace}
\newcommand{\PREF}{\ensuremath{{\algo{PREF}}}\xspace}
\newcommand{\STAG}{\ensuremath{{\algo{STAG}}}\xspace}
\newcommand{\SEMI}{\ensuremath{{\algo{SEMI}}}\xspace}
\newcommand{\STAB}{\ensuremath{{\algo{STAB}}}\xspace}
\newcommand{\COMP}{\ensuremath{{\algo{COMP}}}\xspace}
\newcommand{\tw}[1]{\mathit{tw}(#1)}
\DeclareMathOperator{\width}{width}
\DeclareMathOperator{\children}{children}
\newcommand{\TTT}{\ensuremath{\mathcal{T}}}%
\DeclareMathOperator{\type}{type}
\newcommand{\intr}{\textit{int}}
\newcommand{\leaf}{\textit{leaf}}
\newcommand{\rem}{\textit{rem}}
\newcommand{\join}{\textit{join}}
\newcommand{\SEM}{\ensuremath{\mathcal{S}}\xspace}
\newcommand{\ALL}{\ensuremath{\mathrm{ALL}}\xspace}

\DeclareMathOperator{\adef}{def}
\DeclareMathOperator{\stable}{stable}
\DeclareMathOperator{\stage}{stage}
\DeclareMathOperator{\semistable}{semi-stable}
\DeclareMathOperator{\preferred}{preferred}
\DeclareMathOperator{\complete}{complete}
\DeclareMathOperator{\admissible}{admissible}
\DeclareMathOperator{\conflictfree}{conflict-free}

\DeclareMathOperator{\orig}{origins}
\DeclareMathOperator{\origs}{origins}

\DeclareMathOperator{\pmc}{pc}
\DeclareMathOperator{\ipmc}{ipc}
\DeclareMathOperator{\sipmc}{s-ipc}
\DeclareMathOperator{\bucket}{=_P}%
\DeclareMathOperator{\buckets}{buckets}

\DeclareMathOperator{\post}{post-order}
\DeclareMathOperator{\guess}{States}
\newcommand{\parsi}{\text{parsimonious}}
\newcommand{\subtr}{\text{subtractive}}
\newcommand{\dpa}{\ensuremath{\mathtt{DP}}}
\newcommand{\algo}[1]{\ensuremath{\mathbb{#1}}}
\newcommand{\AlgA}{\algo{A}}%
\newcommand{\mdpa}[1]{\ensuremath{\mathtt{PCNT}_{#1}}}
\newcommand{\MAI}[2]{\ensuremath{#1^+_{#2}}}%
\newcommand{\MAR}[2]{\ensuremath{#1^-_{#2}}}%
\newcommand{\MARR}[2]{\ensuremath{#1^\sim_{#2}}}%
\newcommand{\MAII}[2]{\ensuremath{{#1}^\uplus_{{#2}}}}%
\newcommand{\MAIII}[2]{\ensuremath{{#1}^\oplus_{{#2}}}}%

\newcommand{\numberDotP}{\complexityClassFont{\#\cdot\Ptime}}
\newcommand{\numberDotCoNP}{\complexityClassFont{\#\cdot\co\NP}}
\newcommand{\co}{\complexityClassFont{co}}
\newcommand{\NP}{\complexityClassFont{NP}}
\newcommand{\Ptime}{\complexityClassFont{P}}

\usepackage{microtype}
\usepackage{booktabs}

\title{Counting Complexity for Reasoning in Abstract Argumentation\thanks{Funded by Austrian Science Fund FWF grants I2854, Y698, and P30168-N31, as well as the German Research Fund DFG grants HO 1294/11-1 and ME 4279/1-2. The first two authors are also affiliated with the University of Potsdam, Germany. \longversion{This document is an extended version of a paper that has been accepted for publication at AAAI-19.}}}

\author{Johannes K. Fichte\\TU Dresden\\Int.\ Center for Computational Logic\\Fakult\"at Informatik\\01062 Dresden, Germany\\johannes.fichte@tu-dresden.de
\shortversion{\And}\longversion{\and} Markus Hecher\\TU Wien\\Institute of Logic and Computation\\Favoritenstra{\ss}e 9-11 / E192\\1040 Vienna, Austria\\hecher@dbai.tuwien.ac.at 
\shortversion{\And}\longversion{\and} Arne Meier\\Leibniz Universit\"at Hannover\\Institut f\"ur Theoretische Informatik\\Appelstra{\ss}e 4\\30167 Hannover, Germany\\ meier@thi.uni-hannover.de}

\date{}

\begin{document}

\maketitle

\begin{abstract}
  In this paper, we consider counting and projected model counting of
  extensions in abstract argumentation for various semantics.
  When asking for projected counts we are interested in counting the
  number of extensions of a given argumentation framework while
  multiple extensions that are identical when restricted to the
  projected arguments count as only one projected extension.
  We establish classical complexity results and parameterized
  complexity results when the problems are parameterized by treewidth
  of the undirected argumentation graph.
  To obtain upper bounds for counting projected extensions, we
  introduce novel algorithms that exploit small treewidth of the
  undirected argumentation graph of the input instance by dynamic
  programming (DP). Our algorithms run in time double or triple
  exponential in the treewidth depending on the considered semantics.
  Finally, we take the exponential time hypothesis (ETH) into account
  and establish lower bounds of bounded treewidth algorithms for
  counting extensions and projected extension.
  %
  % In particular, one can not expect (under ETH) to solve #PAS for
  % head-cycle-free or disjunctive programs in polynomial time in the
  % instance size while being single or double exponential in the
  % treewidth, respectively.

\end{abstract}

\section{Introduction}

Abstract argumentation~\cite{Dung95a,Rahwan07a} is a central framework
for modeling and the evaluation of arguments and its reasoning with
applications to various areas in artificial intelligence
(AI)~\shortversion{\cite{AmgoudPrade09a,RagoCocarascuToni18a}}\longversion{\cite{AmgoudPrade09a,DunneBench-Capon05a,Maher16a,McBurneyParsonsRahwan11,RagoCocarascuToni18a}}.
%
% such as legal
% reasoning~\cite{DunneBench-Capon05a},
% e-governance~\cite{AmgoudPrade09a}, and multi-agent
% systems~\cite{McBurneyParsonsRahwan11}.
%
%
The semantics of argumentation is described in terms of arguments that
are acceptable with respect to an abstract framework, such as stable
or admissible. Such arguments are then called extensions of a
framework.
In argumentation, one is particularly interested in the credulous
or skeptical reasoning problem, which asks, given an argumentation framework and an
argument, whether the argument is contained in some or all extension(s) of the
framework, respectively.
A very interesting, but yet entirely unstudied question in abstract
argumentation is the computation and the computational complexity of
counting, which asks for outputting the number of extensions with
respect to a certain semantics. By counting extensions, we can answer
questions such as how many extensions are available containing certain
arguments.
An even more interesting question is how many extensions containing
certain arguments exist when restricted to a certain subset of the
arguments
%arguments are available
%to an attacker if he is interested only in attacking a certain subset
%of the arguments, 
which asks for outputting the number of projected
extensions.

Interestingly, the computational complexity of the decision problem is
already quite hard. More precisely, the problem of credulous acceptance,
which asks whether a given argument is contained in at least one
extension, is \NP-complete for the stable semantics and even
$\Sigma^p_2$-complete for the semi-stable
semantics~\cite{DunneBench-Capon02a,DvorakWoltran10,Dvorak12a}.
The high worst-case complexity is often a major issue to establish
algorithms for frameworks of abstract argumentation. A classical way
in parameterized complexity and algorithmics is to identify structural
properties of an instance and establish efficient algorithms under
certain structural restrictions~\cite{CyganEtAl15}. Usually, we aim
for algorithms that run in time polynomial in the input size and
exponential in a measure of the structure, so-called
\emph{fixed-parameter tractable} algorithms. Such runtime results
require more fine-grained runtime analyses and more evolved reductions
than in classical complexity theory where one considers only the size of
the input.
Here, we take a graph-theoretical measure of the undirected graph of
the given argumentation framework into account. As measure we take
treewidth, which is arguably the most prominent graph invariant in
combinatorics of graph theory and renders various graph problems
easier if the input graph is of bounded treewidth.

Our results are as follows:

%Therefore, we establish the following contributions:
\begin{itemize}
\item We establish the classical complexity of counting extensions and
  counting projected extensions for various semantics in abstract
  argumentation.
\item We present an algorithm that solves counting projected
  extensions by exploiting treewidth in runtime double exponential in
  the treewidth or triple exponential in the treewidth depending on
  the considered semantics.
\item Assuming the exponential time hypothesis (ETH), which
  states that there is some real~$s > 0$ such that we cannot decide
  satisfiability of a given 3-CNF formula~$\phi$ in
  time~$2^{s\cdot\Card{\phi}}\cdot\CCard{\phi}^{\mathcal{O}(1)}$, we
  show that one \emph{cannot} count projected extensions %in time
  double exponentially in the treewidth.
%and an extended version of the hypothesis. 
\end{itemize}

\medskip
\noindent
\textbf{Related work.}  
%\noindent\textbf{Related work.}  
\citex{BaroniDunneGiacomin10} considered general extension counting
and show \#P-completeness and identify tractable cases. We generalize
these results to the reasoning problems.
\citex{LampisMitsouMengel18} considered bounded treewidth algorithms
and established lower bounds for the runtime of an
algorithm %among other problems in AI
that solves \longversion{credulous or skeptical} reasoning in abstract
argumentation under the admissible and preferred semantics.
These results do not trivially extend to counting and are based on
reductions to QBF. 
\longversion{%
  While these reductions yield asymptotically tight bounds, they still
  involve a constant factor.
}%
\shortversion{%
  They yield asymptotically tight bounds, but still involve a constant
  factor.
}%
 Unfortunately, already a small increase even by one can amount
to one order of magnitude in inference time with \emph{dynamic
  programming (DP)} algorithms for QBF. As a result, a factor of just
two can already render it impractical.
\citex{FichteEtAl18} gave DP algorithms
for projected \#SAT and established that it cannot be solved in
runtime double exponential in the treewidth under ETH using results by
\citex{LampisMitsou17}, who established lower bounds for the problem
$\exists\forall$-\SAT. % and bounded treewidth assuming ETH.
\citex{DvorakPichlerWoltran12} introduced DP algorithms that exploit
treewidth to solve decision problems of various semantics in abstract
argumentation. We employ these results and lift them to projected
counting.
Further, DP algorithms for projected counting in answer set
programming (ASP) were recently presented
\cite{FichteHecher18a,FichteHecher18b}.

\section{Formal Background}\label{sec:prelims}
We use graphs and digraphs as usually defined~\cite{BondyMurty08}
and follow standard terminology in computational
complexity~\cite{Papadimitriou94} and parameterized
complexity~\cite{CyganEtAl15}.
Let $\Sigma$ and $\Sigma'$ be some finite alphabets and
$L \subseteq \Sigma^* \times \Nat$ % and
% $L' \subseteq {\Sigma'}^*\times \Nat$
be a parameterized problem. For $(I,k) \in L$, we call
$I \in \Sigma^*$ an \emph{instance} and k the parameter.
%
%
% An \emph{fpt-reduction} $r$ from $L$ to $L'$ is a many-to-one
% reduction from $\Sigma^*\times \Nat$ to ${\Sigma'}^*\times \Nat$
% such that for all $I \in \Sigma^*$ we have $(I,k) \in L$ if and only
% if $r(I,k)=(I',k')\in L'$ with $k' \leqslant g(k)$ for a fixed computable
% function $g: \Nat \rightarrow \Nat$, and~$r$ is computable in time
% $O(f(k)\CCard{I}^c)$ for a computable function $f$ and constant $c$.
%
%
For a set~$X$, let $2^X$ consist of all subsets of~$X$.
Later we use the generalized combinatorial inclusion-exclusion
principle, which allows to compute the number of elements in the union
over all subsets~\cite{GrahamGrotschelLovasz95a}.
%

% $\Card{\bigcup^n_{j = 1} X_j} = \sum_{I \subseteq \{1, \ldots, n\},
% I \neq \emptyset} (-1)^{\Card{I}-1} \Card{\bigcap_{i \in I} X_i}$.

% Given some integer~$n$ and a family of finite subsets~$X_1$, $X_2$,
% $\ldots$, $X_n$. Then, the generalized combinatorial
% inclusion-exclusion principle~\cite{GrahamGrotschelLovasz95a} states
% that the number of elements in the union over all subsets is
% $\Card{\bigcup^n_{j = 1} X_j} = \sum_{I \subseteq \{1, \ldots, n\}, I
%   \neq \emptyset} (-1)^{\Card{I}-1} \Card{\bigcap_{i \in I} X_i}$.

\medskip
\noindent\textbf{Counting Complexity.}
We follow standard terminology in this area \cite{DurandHermannKolaitis05,HemaspaandraVollmer95a}.
In particular, we will make use of complexity classes preceded with the sharp-dot operator `$\#\cdot$'.
If $\mathcal C$ is a decision complexity class then $\#\cdot\mathcal C$ is the class of all counting problems whose  witness function\footnote{A \emph{witness} function is a function $w\colon\Sigma^*\to\mathcal P^{<\omega}(\Gamma^*)$, where $\Sigma$ and $\Gamma$ are alphabets, mapping to a finite subset of $\Gamma^*$. Such functions associate with the counting problem ``given $x\in\Sigma^*$, find $|w(x)|$''.} $w$ satisfies (1.) $\exists$ polynomial $p$ such that for all $y\in w(x)$, we have that $|y|\leqslant p(|x|)$, and (2.) the decision problem ``given $x$ and $y$, is $y\in w(x)$?'' is in $\mathcal C$.
A \emph{parsimonious} reduction between two counting problems $\#A,\#B$ preserves the cardinality between the corresponding witness sets and is computable in polynomial time.
A \emph{subtractive} reduction between two counting problems $\#A$ and $\#B$ is composed of two functions $f,g$ between the instances of $A$ and $B$ such that $B(f(x))\subseteq B(g(x))$ and $|A(x)|=|B(g(x))|-|B(f(x))|$, where $A$ and $B$ are respective witness functions.

\medskip
\noindent\textbf{Abstract Argumentation.}
We consider the Argumentation Framework by \citex{Dung95a}.
% defined as follows.
%
%\begin{definition}
An \emph{argumentation framework~(AF)}, or framework for short, is a
directed graph~$F=(A, R)$ where $A$ is a set of
arguments\footnote{This paper only considers non-empty and finite
  arguments~$A$.} and $R \subseteq A\times A$ a pair arguments
representing direct attacks%
\footnote{Given~$S,S'\subseteq A$. Then, $S\rightarrowtail_R S'$
  denotes $\{s\in S \mid (\{s\}\times S')\cap R \neq \emptyset\}$, and
  $S\leftarrowtail_R S'\eqdef \{s\in S \mid (S' \times \{s\}) \cap
  R\neq \emptyset\}$.}  of arguments.
%\end{definition}
%
% \begin{example}
%   Figure~\ref{fig:argu} provides an argumentation
%   framework~$F_{\mathit{Ex}}=(A_{F_{\mathit{Ex}}},
%   R_{F_{\mathit{Ex}}})$.
% \end{example}
%
%Typically, the graph instances are considered to be given and 
%From now on, we assume a given framework~$F=(A,R)$ and
%
In argumentation, we are interested in computing so-called
\emph{extensions}, which are subsets~$S \subseteq A$ of the arguments
that meet certain properties according to certain semantics as given
below.
%
% In this paper, we are interested in the semantics
% \emph{conflict-free}, \emph{admissible}, \emph{complete},
% \emph{preferred}, \emph{semi-stable} and \emph{stable}.
%
%\begin{definition}
An argument~$s \in S$, is called \emph{defended by $S$ in $F$} if for
every $(s', s) \in R$, there exists $s'' \in S$ such that
$(s'', s') \in R$.  The family~$\adef_F(S)$ is defined by
$\adef_F(S) \eqdef\SB s \SM s \in A, s \text{ is defended by $S$ in
  $F$} \SE$
%
%, and any $S' \subseteq \text{def}_F(S)$ is called \emph{defended by $S$ in $F$}.
%
We say $S \subseteq A$ is \emph{conflict-free in~$S$} if
$(S\times S) \cap R = \emptyset$; $S$ is \emph{admissible in $F$} if
(i) $S$ is \emph{conflict-free in $F$}, and
(ii) every $s \in S$ is \emph{defended by $S$ in $F$}.
Assume an admissible set~$S$.
Then,
(iiia) $S$ is \emph{complete in~$F$} if $\adef_F(S) = S$;
(iiib) $S$ is~\emph{preferred in~$F$}, if there is no $S' \supset S$ that is \emph{admissible in $F$};
(iiic) $S$ is \emph{semi-stable in $F$} if there is no admissible set $S' \subseteq A$ in~$F$ with~$S^+_R\subsetneq (S')^+_R$ where $S^+_R:=S\cup\SB a\SM (b,a)\in R, b \in S \SE$;
% is the \emph{range of $S$ in $F$}.
%
% $S^+_R:=S\cup\{\,a\mid\exists b \in S\ s.t.\ (b,a)\in R\,\}$ is the
% \emph{range of $S$ in $F$}.
%
%}%
%
% with~$S^+_R\subset (S')^+_R$%
% %
% \footnote{
% %
%   $S^+_R:=S\cup\SB a\SM (b,a)\in R, b \in S \SE$ is the \emph{range of
%     $S$ in $F$}.
%
%   % $S^+_R:=S\cup\{\,a\mid\exists b \in S\ s.t.\ (b,a)\in R\,\}$ is the
%   % \emph{range of $S$ in $F$}.
% %
% }%
%
%
(iiid) $S$ is \emph{stable in~$F$} if % $S \subseteq A$ is called \emph{stable} if
% (i) $S$ is \emph{conflict-free in $F$}, and
every $s \in A \setminus S$ is \emph{attacked} by some $s' \in S$.
A conflict-free set~$S$ is \emph{stage in $F$} if there is no conflict-free set~$S'\subseteq A$ in~$F$ with~$S^+_R\subsetneq (S')^+_R$.
Let $\ALL$ abbreviate the set $\{$admissible, complete, preferred, semi-stable, stable, stage$\}$.
For a semantics~$\SEM \in \ALL$, $\SEM(F)$ denotes the set of
\emph{all extensions} of semantics~\SEM in $F$.
%
% $\bigcup_{S \subseteq A:\ S\ \text{is SEM}}\ \{S\}$.
%
In general
$\stable(F) \subseteq \semistable(F) \subseteq \preferred(F)$
$\subseteq \complete(F) \subseteq \admissible(F) \subseteq
\conflictfree(F)$ and~$\stable(F) \subseteq \stage(F) \subseteq \conflictfree(F)$.
%
%

% If framework $F$ is clear from the context we will it.  omit
% framework
% $F$. %, for instance instead of writing that $S$ is admissible in $F$,
% we will denote for short that $S$ is admissible.

\medskip
\noindent
\textbf{Problems of Interest.}
In argumentation one is usually interested in credulous and skeptical
reasoning problems. In this paper, we are in addition interested in
counting versions of these problems.
Therefore, let $\SEM \in \ALL$ be an abstract argumentation
semantic, $F=(A,R)$ be an argumentation framework, and $a \in A$ an
argument. 
The \emph{credulous reasoning problem} $\Cred_\SEM$ asks to decide
whether there is an \SEM-extension of $F$ that contains the
(credulous) argument~$a$. The \emph{skeptical reasoning problem} $\Skep_\SEM$ asks
to decide whether all \SEM-extensions of $F$ contain the argument~$a$.
%
%\optional{The \emph{extension counting problem} $\numberExt_\SEM$ asks
%  to output the number of all extensions of~$F$,~i.e.,
%  $\Card{ \SB S \SM S \in \SEM(F) \SE}$.}
%
The \emph{credulous counting problem $\numberCred_\SEM$} asks to output the
number of \SEM-extensions of~$F$ that contain~$a$,~i.e.,
$\Card{ \SB S \SM S \in \SEM(F), a \in S \SE}$.
%
%
%
%
% FUTURE USAGE
% 
%
% An instance of the projected extension counting problem is a
% pair~$(F,P)$ where $F$ is an argumentation framework~$F=(A,R)$ and $P$
% is a subset~$P \subseteq A$ of arguments.  We call the set~$P$
% \emph{projection arguments} of the instance. 
%
%
%
%
% unused
%
% The \emph{projected extension count} of the framework~$F$ with
% respect to~$P$ is the number of \SEM-extensions restricted to the
% projection arguments~$P$, i.e.,
%
%
The \emph{projected credulous counting problem ($\PCC_\SEM$)} asks to
output the number of $\SEM$-extensions restricted to the projection
arguments~$P$,~i.e.,
$\Card{ \SB S \cap P \SM S \in \SEM(F), a \in S \SE}$.
\longversion{%
  One can view \PCC as a generalization of~$\numberCred_S$.
}%
\begin{figure}
	\centering
	\begin{tikzpicture}[every node/.style={draw=black,thin}]
		%\node (u) {$u$};
		\node (swim) [] {\textbf{d}rinking cocktails};
		\node (ski) [right=2em of swim] {\textbf{s}urfing};
		\node (wexp) [right=1.5em of ski] {surfing \textbf{e}xpensive};
		\node (regular) [above=0.5em of ski,xshift=-3em] {seasonal surfing \textbf{p}ass req.};
		\node (once) [above=0.5em of wexp] {\textbf{c}heap if once};
		\node (adventure) [below=0.5em of swim] {\textbf{a}dventure req.};
		\node (chill) [below=0.5em of ski] {\textbf{r}elaxing req.};
		
		\draw [-stealth'] (wexp) to (ski);
		\draw [-stealth'] (swim) to (regular);
		\draw [-stealth'] (regular) to (wexp);
		\draw [stealth'-stealth'] (regular) to (once);
		\draw [-stealth'] (once) to (wexp);
		\draw [stealth'-stealth'] (adventure) to (chill);
		\draw [-stealth'] (adventure) to (swim);
		\draw [-stealth'] (chill) to (ski);
		%\draw [stealth'-stealth'] (w) -- (x);
		%\draw [-stealth'] (w) to [bend right=45] (y);
		%\path[-stealth'] (z) edge [loop above, >=stealth'] ();
	\end{tikzpicture}
	\caption{Argumentation framework $F$: surfing vs. cocktails.}
	\label{fig:decision}
\end{figure}
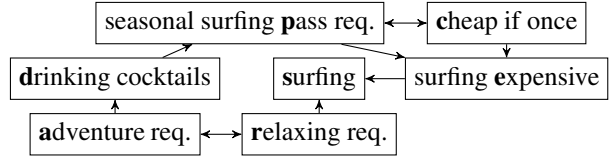
\begin{example}
Consider framework~$F$ from Figure~\ref{fig:decision}, which depicts a framework
for deciding between surfing and drinking cocktails.
Framework~$F$ admits three stable extensions~$\stable(F)=\{\{d, r, c\},
\{s, a, c\}, \{s, a, p\}\}$. \longversion{Then, }$\numberCred_{\stable}$ for argument~$s$ equals~$2$,
whereas $\PCC_{\stable}$ for argument~$s$ restricted to~$P\eqdef\{a,r\}$ equals~$1$.
%
%       >          seasonal pass necessary for price/value	<->	doing things one time is not expensive
%--|	|			          | 							| |
%| |>   |				  >							| |
%| swimming           snowboarding   <- winter sport expensive		         	 <------| |
%|   ^		        ^       ^								  |
%|   |		        |       |								  |
%|  adventureness <-> relaxing  weatherproof  <-> weatherproof equipment avail			  |
%|				 ^								  |
%|				 |								  |
%|->				apres ski as plan b		<-- drinking is expensive        <|
%

%swimming, relaxing
%swimming, relaxing, equipment avail
%swimming, relaxing, weatherpoof

%swimming, relaxing, equipment avail, doing things one time not expensive
%swimming, relaxing, weatherproof, doing things one time not expensive

%snowboarding, adventureness, equipment avail, seasonal pass necessary
%snowboarding, adventureness, equipment avail, doing thins one tim not expensive

%snowboarding, adventureness, equipment avail, drinking expensive, seasonal pass necessary
%snowboarding, adventureness, equipment avail, apres ski plan b, doing thins one tim not expensive

\end{example}

\medskip
\noindent\textbf{Tree Decompositions (TDs).} %
%\noindent\textbf{Tree Decompositions (TDs).}
For a tree~$T=(N,A,n)$ with root~$n$ and a node~$t \in N$, we let
$\children(t, T)$ be the sequence of all nodes~$t'$ in arbitrarily but
fixed order, which have an edge~$(t,t') \in A$.
Let $G=(V,E)$ be a graph.
A \emph{tree decomposition (TD)} of graph~$G$ is a pair
$\TTT=(T,\chi)$, where $T=(N,A,n)$ is a rooted tree, $n\in N$ the root,
and $\chi$ a mapping that assigns to each node $t\in N$ a set
$\chi(t)\subseteq V$, called a \emph{bag}, such that the following
conditions hold:
% (we refer to the vertices of $T$ as \emph{nodes} to make the
% distinction between $T$ and $G$ clearer).
(i) $V=\bigcup_{t\in N}\chi(t)$ and
% for every edge~$uv \in E(G)$ there exists some~$t$ such
% that~$uv \in \chi(t)$
$E \subseteq\bigcup_{t\in N}\SB \{u,v\} \SM u,v\in \chi(t)\SE$; and (ii)
for each $r, s, t$, such that $s$ lies on the path from $r$ to
$t$, we have $\chi(r) \cap \chi(t) \subseteq \chi(s)$.
Then, $\width(\TTT) \eqdef \max_{t\in N}\Card{\chi(t)}-1$.  The
\emph{treewidth} $\tw{G}$ of $G$ is the minimum $\width({\TTT})$ over
all tree decompositions $\TTT$ of $G$.
For arbitrary but fixed $w \geq 1$, it is feasible in linear time to
decide if a graph has treewidth at most~$w$ and, if so, to compute a
TD of width $w$~\cite{Bodlaender96}.
In order to simplify case distinctions in the algorithms, \emph{we
  assume nice TDs}, which can be computed in linear time without
increasing the width~\cite{Kloks94a} and are defined as follows.
For a node~$t \in N$, we say that $\type(t)$ is $\leaf$ if
$\children(t,T)=\langle \rangle$; $\join$ if
$\children(t,T) = \langle t',t''\rangle$ where
$\chi(t) = \chi(t') = \chi(t'') \neq \emptyset$; $\intr$
(``introduce'') if $\children(t,T) = \langle t'\rangle$,
$\chi(t') \subseteq \chi(t)$ and $|\chi(t)| = |\chi(t')| + 1$; $\rem$
(``remove'') if $\children(t,T) = \langle t'\rangle$,
$\chi(t') \supseteq \chi(t)$ and $|\chi(t')| = |\chi(t)| + 1$. If for
every node $t\in N$, $\type(t) \in \{ \leaf, \join, \intr, \rem\}$ and
bags of leaf nodes and the root are empty, then the TD is called
\emph{nice}.
%

%\begin{example}
%  Figure~\ref{fig:graph-td} illustrates a tree decomposition of~$F_1$
%  (see Figure~\ref{fig:argu}) of width~$2$.
%  %By a basic
%  %property\footnote{The vertices $e$,$b$,$d$ are all neighbors to
%  %  each other in~$G_1$.} of tree decompositions~\cite{Kloks94a}, the
%  %treewidth of~$G_1$ is~$2$.
%\end{example}

\section{Classical Counting Complexity}
In this section, we investigate the classical counting complexity of the credulous reasoning problem.  

\begin{lemma}[$\star$\footnote{Proofs of marked statements (``$\star$'') are ommitted or shortened.}]\label{lem:numberext-in-numberp}
$\numberCred_\SEM$ is in\\
 (1) $\numberDotP$ if $\SEM\in\{$conflict-free, stable, admissible, complete$\}$ \longversion{and\\}
 (2) $\numberDotCoNP$ if $\SEM\in\{$preferred, semi-stable, stage$\}$.
\end{lemma}

The next lemma does not consider conflict-free extensions.

\begin{lemma}[$\star$]\label{lem:numberext-numberp-hard}
  $\numberCred_\SEM$ is\\
  (1) $\numberDotP$-hard under $\parsi$ reductions if \shortversion{\\
  ~\hbox{~}\hspace{1em}}$\SEM\in\{$stable, admissible, complete$\}$ and\\
  (2) $\numberDotCoNP$-hard under $\subtr$ reductions if \shortversion{\\
  ~\hbox{~}\hspace{1em}}$\SEM\in\{$semi-stable, stage$\}$.
\end{lemma}
\begin{proof}[Proof (Sketch)]
(1) 
Start with the case of stable or complete extensions. 
Following the construction of \citex{DunneBench-Capon02a}, we parsimoniously reduce from $\#\SAT$.
For the case of admissible extensions, to count correctly, it is crucial that for
        each $x_i$ either argument $x_i$ or $\bar x_i$ is
        part of the extension. To ensure this, we introduce arguments
        $s_1,\dots,s_n$ attacking $t$ that can only be defended by one
        of $x_i$ or $\bar x_i$. 

(2) The formalism of circumscription is well-established in the area of AI \cite{McCarthy80}. 
Formally, one considers assignments of Boolean formulas that are \emph{minimal} regarding the \emph{pointwise partial order} on truth assignments:
if $s=(s_1,\dots,s_n),s'=(s_1',\dots,s_n')\in\{0,1\}^n$, then write $s<s'$ if $s\neq s'$ and $s_i\leqslant s_i'$ for every $i\leqslant n$.
Then, we define the problem $\countCirc$ which asks given a Boolean formula $\varphi$ in CNF to output the number of minimal models of $\varphi$.
\citex{DurandHermannKolaitis05} showed that $\countCirc$ is
$\numberDotCoNP$-complete via subtractive reductions (a generalization of parsimonious reductions).  
The crux is, that choosing negative literals is more valuable than selecting positive ones.
This is achieved by adding additionally attacked arguments to each negative literal. 
\end{proof}

Lemma~\ref{lem:numberext-in-numberp} and \ref{lem:numberext-numberp-hard} together show the following theorem.

\begin{theorem}
  $\numberCred_\SEM$ is\\
  (1) $\numberDotP$-complete under $\parsi$ reductions if \shortversion{\\
  ~\hbox{~}\hspace{1em}}$\SEM\in\{$stable, admissible, complete$\}$ and\\
  (2) $\numberDotCoNP$-complete under $\subtr$ reductions if \shortversion{\\
  ~\hbox{~}\hspace{1em}}$\SEM\in\{$semi-stable, stage$\}$.
\end{theorem}

Now, consider the case of projected counting.

\begin{lemma}[$\star$]\label{lem:proj-upperbound}
  $\PCC_\SEM$ is in\\
  (1) $\#\cdot\NP$ if $\SEM\in\{$stable, admissible, complete$\}$,
  and\\
  (2) $\#\cdot\Sigma_2^\Ptime$ if $\SEM\in\{$semi-stable, stage$\}$.
\end{lemma}
\begin{proof}[Proof (Sketch)]
	Given a framework, a projection set, and an argument $a$.
	We non-deterministically branch on a possible projected extension $S$.
	Accordingly, we have $S\subseteq P$.
	If $a\in S$ and $S$ is of the respective semantics, then we accept.
	Otherwise we make a non-deterministic guess $S'\supseteq S$, verify if $P\cap S'=S$, $a\in S'$, and $S'$ is of the desired semantics.
	Extension verification is for (1) in $\Ptime$, and for (2) in $\co\NP$.
	Concluding, we get an $\NP$ oracle call for the first case, and an $\NP^{\co\NP}=\NP^\NP=\Sigma_2^\Ptime$ oracle call in the second case.
\end{proof}

Consider the problem $\#\Sigma_k\SAT$, which asks, given
$\varphi(Y)=\exists x_1\forall x_2\cdots Q_kx_k\psi(X_1,\dots,X_k,Y)$,
where $\psi$ is a propositional DNF if $k$ is even (and CNF if $k$ is
odd), $X_i$, for each $i$, and $Y$ are sets of variables, to output the
number of truth assignments to the variables from $Y$ that satisfy
$\varphi$.
\citex{DurandHermannKolaitis05} have shown that the problem
%\countproblem{$\#\Sigma_k\SAT$}{$\varphi(Y)=\exists x_1\forall
%  x_2\cdots Q_kx_k\psi(X_1,\dots,X_k,Y)$, where $\psi$ is a
%  propositional DNF if $k$ is even (and CNF if $k$ is odd), $X_i$, for
%  each $i$, and $Y$ are sets of variables.}{Number of truth
%  assignments to the variables from $Y$ that satisfy $\varphi$.}
is $\#\cdot\Sigma^\Ptime_k$-complete via parsimonious reductions.

\begin{lemma}[$\star$]\label{lem:pcct-numberdotcnp-hard}
  $\PCC_\SEM$ is\\
  (1) $\#\cdot\Sigma_2^\Ptime$-hard w.r.t.\ $\parsi$ reductions if \shortversion{\\
  ~\hbox{~}\hspace{1em}}$\SEM\in\{$stage, semi-stable$\}$ and\\
  (2) $\#\cdot\NP$-hard w.r.t.\ $\parsi$ reductions if \shortversion{\\
  ~\hbox{~}\hspace{1em}}$\SEM\in\{$admissible, stable, complete$\}$.
\end{lemma}
\begin{proof}[Proof (Sketch)]
(1) We state a parsimonious reduction from $\#\Sigma_2\SAT$ to $\PCC_\SEM$.
We use an extended version of the construction of \citex{DvorakWoltran10}.
Given a formula $\varphi(X)=\exists Y\forall Z\;\psi(X,Y,Z)$, where $X,Y,Z$ are sets of variables, and $\psi$ is a DNF.
Consider now the negation of $\varphi(X)$, i.e., $\varphi'(X)=\lnot\varphi(X)\equiv\forall Y\exists Z\; \lnot\psi(X,Y,Z)$.
Let $\psi'(X,Y,Z)$ be $\lnot\psi(X,Y,Z)$ in NNF.
Accordingly, $\psi'$ is a CNF, $\psi'(X,Y,Z)= \bigwedge_{i=1}^m C_i$ and $C_i$ is a disjunction of literals for $1\leqslant i\leqslant m$.
Note that, the formula $\varphi'(X)$ is of the same kind as the formula in the construction of \citex{DvorakWoltran10}.
Now define an argumentation framework $AF=(A,R)$, where
	\begin{align*}
		A &= \SB x, \bar x\SM x\in X\SE\cup \SB y, \bar y, y', \bar y'\SM y\in Y \SE\\
		  &\;\cup\SB z,\bar z\SM z\in Z\SE\cup\{t,\bar t, b\}\displaybreak[1]\\
		R &= \SB(y',y'), (\bar y',\bar y'), (y,y'), (\bar y,\bar y'), (y,\bar y), (\bar y,y)\SM y\!\in\!Y\}\\
		&\;\cup\{(b,b),(t,\bar t),(\bar t,t), (t,b)\}\cup\SB(C_i,t) \SM 1\leqslant i\leqslant m \SE\\
		&\;\cup\SB(u,C_i) \SM u\in X\cup Y\cup Z, u\in C_i, 1\leqslant i\leqslant m \SE\\
		&\;\cup\SB(\bar u,C_i) \SM z\in X\cup Y\cup Z, \bar u\in C_i, 1\leqslant i\leqslant m \SE
	\end{align*}
Note that, by construction, the $y',\bar y'$ variables make the extensions w.r.t.\ the universally quantified variables $y$ incomparable.
Further observe that choosing $t$ is superior to selecting $\bar t$, as $t$ increases the range by one more.
If for every assignment over the $Y$-variables there exists an assignment to the $Z$-variables, then, each time, when there is a possible solution to $\psi'(X,Y,Z)$, so semantically $\neg\psi(X,Y,Z)$, w.r.t.\ the free $X$-variables, the extension will contain $t$.
As a result, the extensions containing $t$ correspond to the unsatisfying assignments.
Let $A(\varphi(X))$ be the set of assignments of a given $\#\Sigma_2\SAT$-formula, and $B(AF,P,a)$ be the set of stage/semi-stable extensions which contain $a$ and are projected to $P$.
Then, one can show that $\Card{A(\varphi(X))}=\Card{B(AF,X,\bar t)}$ proving the desired reduction (as $\bar t$ together with the negation of $\varphi(X)$ in the beginning, intuitively, is a double negation yielding a reduction from $\#\Sigma_2\SAT$).

(2) Now turn to the case of admissible, stable, or complete extensions.
Again, we provide a similar parsimonious reduction, but this time, from $\#\Sigma_1\SAT$ to $\PCC_\SEM$.
\end{proof}%

\begin{theorem}
	$\PCC_\SEM$ is\\
        (1) $\#\cdot\NP$-complete via $\parsi$ reductions if \shortversion{\\
        ~\hbox{~}\hspace{1em}}$\SEM\in\{$stable, admissible, complete$\}$, and\\
        (2) $\#\cdot\Sigma_2^\Ptime$-complete via $\parsi$ reductions if \shortversion{\\
        ~\hbox{~}\hspace{1em}}$\SEM\in\{$stage, semi-stable$\}$.
\end{theorem}

Similarly, one can introduce problems of the form $\#\problemFont{Skep}_\SEM$ and $\#\problemFont{PSkep}_\SEM$ corresponding to the counting versions of the skeptical reasoning problem.
As skeptical is dual to credulous reasoning, one easily obtains completeness results for the dual counting classes.

\section{DP for Abstract Argumentation}
In this section, we recall DP techniques from the literature to solve
skeptical and credulous reasoning in abstract
argumentation. Additionally, we establish lower bounds for exploiting
treewidth in algorithms that solve these
problems %skeptical or credulous reasoning
for the most common semantics.
Therefore, let $F=(A,R)$ be a given argumentation framework and $\SEM$
be an argumentation semantics.
While an abstract argumentation framework can already be seen as a
digraph, treewidth is a measure for undirected graphs. Consequently, we
consider for framework~$F$ the \emph{underlying graph}~$G_F$, where we
simply drop the direction of every edge,~i.e., $G_F=(A,R')$ where
$R' \eqdef \SB \{u,v\} \SM (u,v) \in R\SE$.
Let $\TTT = (T, \chi)$ be a TD of the underlying graph of~$F$.
Further, we need some auxiliary definitions. 
Let $T = (N,\cdot,n)$ and $t \in N$.  Then, $\post(T,n)$ defines a
sequence of nodes for tree~$T$ rooted at~$n$ in \emph{post-order}
traversal.
The \emph{bag-framework} is defined as $F_t \eqdef (A_t, R_t)$,
where~$A_t\eqdef A\cap\chi(t)$ and
$R_t \eqdef (A_t \times A_t) \cap R$, the \emph{framework below $t$}
as $F_{\leqslant t}\eqdef (A_{\leqslant t}, R_{\leqslant t})$, where
$A_{\leqslant t} \eqdef \SB a \SM a \in \chi(t), t' \in \post(T,t) \SE$,
%
%$A_{\leqslant t} \eqdef \bigcup_{t' \in \post(T,t)} \chi(t)$,
and~$R_{\leqslant t}\eqdef (A_{\leqslant t} \times A_{\leqslant t}) \cap R$.
%
% UNUSED
%
% Finally, the \emph{program strictly below $t$} is defined as
% $F_{< t}\eqdef (A_{< t}, R_{<t})$,
% where~$A_{<t}\eqdef A_{\leqslant t}\setminus A_t$ and
% $R_{<t}\eqdef R_{\leqslant t}\setminus R_t$.
%
It holds that $F_n = F_{\leqslant n} = F$.
%
%\begin{example}
%Consider program~$\prog$ from Example~\ref{ex:running} and the tree decomposition of~$G_\prog$ in Figure~\ref{fig:graph-td}. 
%Observe that~$\prog_{t_1}=\{r_2\}$, 
%$\prog_{t_2}=\{r_3,r_4,r_5\}$,
%$\prog_{t_3} = \{r_1\}$, but~$\progt{t_3}=\prog$.
%\end{example}
%
% For an
% example we refer to Example~\ref{ex:bagprog}$^\star$.
%

A standard approach~\cite{BodlaenderKloks96} to benefit
algorithmically from small treewidth is to design DP algorithms, which
traverse a given TD and run at each node a so-called \emph{local
  algorithm~$\AlgA$}.
The local algorithm does a case distinctions based on the
types~$\type(t)$ of a nice TD and stores information in a table, which
is a set of rows where a \emph{row}~$\vec u$ is a sequence of fixed
length (and the length is bounded by the treewidth).
%
%
% We abbreviate the \emph{empty table} by~$\epsilon$, i.e.,
% $\epsilon = \{\}$.
%
%
% The actual length, content, and meaning of the rows depend on the
% algorithm~$\AlgA$.  
% %
% %
% Since we later traverse the tree decomposition
% repeatedly running different algorithms, we explicitly state
% \emph{$\AlgA$-row} if rows of this \emph{type} are syntactically used
% for algorithm~$\AlgA$ and similar %for sequences and tables.
% \emph{$\AlgA$-table} for tables.
%
%
Later, we traverse the TD multiple times. We access also information
in tables computed in previous traversals and formalize access to
previously computed tables in \emph{tabled tree decomposition (TTD)}
by taking in addition to the TD~$\TTT=(T,\chi)$ a mapping $\tau$ that
assigns to a node~$t$ of~$T$ also a table. Then, the TTD is the triple
$\TTT=(T,\chi,\tau)$.  Later, for simple use in algorithms, we assume
$\tau(t)$ is initialized by the empty set for every node~$t$ of~$T$.
%
%
% not relevant any more
%
% In particular, the mapping~$\tau$ of an $\AlgA$-TTD is a
% total mapping on the nodes of~$T$ to~$\AlgA$-tables.
%
% This just gives us a shortcut to talk about tables
%
% \begin{figure}[t]
% \centering
% \includegraphics[scale=0.8]{figure.pdf}
% \caption{The DP approach, where table algorithm~$\algo{A}$ modifies
%   tables.~\cite{FichteEtAl17a}}
% \label{fig:framework}
% \end{figure}%
%
%
% In the context of dynamic programming, we usually restrict the
% interpretation to the content of the currently considered bag in
% order to obtain nice runtime bounds.
%
To solve the considered problem, we run the following steps:
\begin{enumerate}%[leftmargin=*]
\item Compute a TD~$\TTT=(T,\chi)$ of the underlying graph %~$G_F$
  of~$F$.
  %of the given argumentation framework~$F$.
\item\label{step:dp} Run algorithm~$\dpa_\AlgA$, which 
  %\begin{itemize}
  %\item 
  takes a TTD~$\mathcal{T}=(T,\chi, \iota)$ with~$T=(N,\cdot,n)$ and
  traverses~$T$ in post-order.
  %\item 
  At each node~$t \in N$ it stores the result of algorithm~$\AlgA$ in
  table~$o(t)$. Algorithm~$\AlgA$ can access only information that is
  restricted to the currently considered bag, namely, the type of the
  node~$t$, the atoms in the bag~$\chi(t)$, the bag-framework~$F_t$,
  and every table~$o(t')$ for any child~$t'$ of $t$.
  % 
  % \footnote{%
  % Note that in Listing~\ref{fig:dpontd}, $\AlgA$ takes in
  % addition as input the set~$P$ and table~$\iota_t$. We will later
  % reuse this listing. Then, $P$ represents the set of projected
  % atoms and $\iota_t$ is a table at~$t$ from an earlier traversal. 
  % %
  % }
\item Print the solution by interpreting table~$o(n)$ for root~$n$ of
  the resulting TTD~$(T, \chi, o)$.
%\end{itemize}
%\item Print the result by interpreting table~$o(n)$ for root~$n$
%  of~$T$.
\end{enumerate}

% \begin{algorithm}[t]%
%   \KwData{%
%     Problem instance~$(F,c,P)$, TTD~$\TTT=(T,\chi, \iota)$
%     of~$G_F$ such that~$n$ is the root of~$T$ and
%     $\children(t, T) = \langle t_1, \ldots, t_\ell\rangle$.
%     %
%   }%
%   % 
%   \KwResult{%
%     $\AlgA$-TTD~$(T,\chi, o)$ with $\AlgA$-table mapping~$o$
%     %
%   } %
  
%   \hspace{-0.1em}$o \leftarrow \text{ empty mapping}$
  
%   %$o \eqdef $ post-order($T$)\;%
%   \For{\text{\normalfont iterate} $t$ in \text{\normalfont post-order}(T,n)}{
%     %
%     \vspace{-0.05em}%
%     %
%     %\hspace{-0.6em}
    
%   %   \hspace{-0.6em}$\Tab{} \eqdef \langle \tau(t_1),\ldots,
%   %   \tau(t_\ell) \rangle$ where
%   %   $\children(t, T) = \langle t_1, \ldots, t_\ell\rangle$%\hspace{-1em}%,\Prev)
%   %   %\rangle\hspace{-5em}$
%   % %
  
%     %\hspace{-0.5em}
%     $o(t) \leftarrow {\AlgA}(t,
%     \chi(t),\iota(t),(F_t,c,P), \langle o(t_1), \ldots,
%     o(t_\ell) \rangle)$\hspace{-0.8em} %
%     %\vspace{-0.5em} %
%   }%
%   \Return{$(T, \chi, o)$}
%   %
%   \captionsetup{format=hang}
%   \caption{%
%     Algorithm ${\dpa}_{\AlgA}((F,c,P), \TTT)$:
%     Dynamic programming on TTD
%     $\mathcal{T}$,~c.f.,~\protect\cite{FichteEtAl17a}.
%   % 
%   }%
% \label{fig:dpontd}
% %\input{flowchart}
% \end{algorithm}%

%DP for ADMISSIBLE extensions
{
	\renewcommand{\eqdef}{\leftarrow}
	%
	% NEW VERSION 
	%
	%\vspace{-5ex}
	 \begin{algorithm}[t]
	   \KwData{Node~$t$, bag $\chi_t$, bag-framework~$F_t=(A_t, R_t)$,
	   credulous argument~$c$, and
	     $\langle \tau_{1}, \tau_2 \rangle$ is the sequence of tables
	     of~children of~$t$. \shortversion{
	     
	     }\textbf{Out: } Table~$\tau_t.$}
	     %
	     %\KwResult{ \ADM-table~$\tau_{t}.$}
	   %
	   \lIf(\hspace{-1em})
	   {$\type(t) = \leaf$}{%
	     $\tau_{t} \eqdef \{ \langle
	     \tuplecolor{\inputPredColor}{\emptyset}, \tuplecolor{\outputPredColor}{\emptyset}, \tuplecolor{\statePredColor}{\emptyset}
	     \rangle \}$\label{line:primleaf}%
	     % 
		%\EndIIf 
	   }%
	  \uElseIf{$\type(t) = \intr$ and $a\hspace{-0.1em}\in\hspace{-0.1em}\chi_t$ is the introduced argum.}{
	   %\vspace{-0.05em}
	   %\makebox[3.9cm][l]{\hspace{-1em}
	   $\hspace{-0.5em}\tau_{t} \eqdef %\compr(
	   \{ \langle \tuplecolor{\inputPredColor}{J}, \tuplecolor{\outputPredColor}{\MAII{O}{A_t \rightarrowtail_{R_t} J}}, \tuplecolor{\statePredColor}{\MAII{{D}}{J \leftarrowtail_{R_t} A_t}} \rangle %}
	     \mid\langle \tuplecolor{\inputPredColor}{I}, \tuplecolor{\outputPredColor}{O}, \tuplecolor{\statePredColor}{{D}} \rangle\in \tau_{1}, J \in \{I, \MAI{I}{a}\},$
	    % \makebox[2.9cm][l]{}
	    $J \rightarrowtail_{R_t} J = \emptyset, J \cap \{c\} = \chi(t) \cap \{c\} \}$
	      %$%)
	      %\hspace{-5em}$
	      \label{line:primintr}
	   \vspace{-0.05em}
	     }\vspace{-0.05em}%
	     \uElseIf{$\type(t) = \rem$ \KwAnd $a \not\in \chi_t$ is the removed argum.}{%
%	       \makebox[3.9cm][l]{%\hspace{-1em}
	       $\hspace{-0.5em}\tau_{t} \eqdef %\compr(
	       \{ \langle \tuplecolor{\inputPredColor}{\MAR{I}{a}}, \tuplecolor{\outputPredColor}{\MAR{{O}}{a}}, \tuplecolor{\statePredColor}{\MAR{{D}}{a}}
	       \rangle %}
	       \mid\langle \tuplecolor{\inputPredColor}{I}, \tuplecolor{\outputPredColor}{{O}}, \tuplecolor{\statePredColor}{{D}}
	       \rangle \in \tau_{1}, a \not\in {O} \setminus {D} \}%)
	       \hspace{-5em}$\label{line:primrem}
	       \vspace{-0.1em}
	     } %
	     \ElseIf{$\type(t) = \join$}{%, and $\Tab{} = \langle \tau', \tau'' \rangle$}{%
	       %\makebox[2.9cm][l]{\hspace{-1em}
	       $\hspace{-0.5em}\tau_{t} \eqdef %\compr(
	       \{ \langle \tuplecolor{\inputPredColor}{I}, \tuplecolor{\outputPredColor}{\MAII{{{O}_1}}{{{O}_2}}}, \tuplecolor{\statePredColor}{\MAII{{{D}_1}}{{{D}_2}}}
		 \rangle %}
		 \mid\langle \tuplecolor{\inputPredColor}{I}, \tuplecolor{\outputPredColor}{{O}_1}, \tuplecolor{\statePredColor}{{D}_1} \rangle\!\!\in\!\! \tau_{1},$
		 $\langle \tuplecolor{\inputPredColor}{I}, \tuplecolor{\outputPredColor}{{{O}_2}}, \tuplecolor{\statePredColor}{{{D}_2}} \rangle\!\!\in\!\! \tau_{2}\}\hspace{-5em}$\label{line:primjoin}
	       \vspace{-0.1em}
	     } 
	     \Return $\tau_{t}$
	     \vspace{-0.25em}
	     \caption{Local algorithm~$\ADM(t, \chi_t, \cdot, (F_t, c, \cdot),
	       \langle \tau_1, \tau_2 \rangle)$, c.f., \cite{DvorakPichlerWoltran12}.}
	 \label{fig:prim}\algorithmfootnote{
	\renewcommand{\eqdef}{{\ensuremath{\,\mathrel{\mathop:}=}}}
	  $\MAII{S}{S'} \eqdef S \cup S'$,
	  $\MAI{S}{e} \eqdef S \cup \{e\}$, and
	  $\MAR{S}{e} \eqdef S \setminus \{e\}$.}
	\end{algorithm}%
	\renewcommand{\eqdef}{{\ensuremath{\,\mathrel{\mathop:}=}}}

}
\medskip
\noindent\textbf{Credulous Reasoning.} %
DP algorithms for credulous reasoning of various semantics have
already been established in the
literature~\cite{DvorakPichlerWoltran12} and their implementations are
also of practical interest~\cite{DvorakEtAl13}.  While a DP algorithm
for semi-stable~\cite{BliemHecherWoltran16a} semantics was presented
as well, stage semantics has been missing. This section fills the gap
by introducing a local algorithm for this case. The worst case
complexity of these algorithms depends on the semantics and ranges
from single to double exponential in the treewidth.
In the following, we take these algorithms from the literature,
simplify them and adapt them to solve~\PCC for the various semantics.
First, we present the algorithm~$\dpa_\ADM$ that uses the algorithm in
Listing~\ref{fig:prim} as local algorithm to solve credulous reasoning
for the admissible semantics. $\dpa_\ADM$ outputs a new TTD that we
use to solve our actual counting problem.
%
% Note that the tree decomposition itself remains the same, but for
% readability, we keep the computed tables and nodes aligned.
%
At each node~$t$, we store in table~$o(t)$ rows of the
form~$\vec u = \langle I, O, D\rangle$ and construct parts of
extensions.
The first position of the rows consists of a set~$I\subseteq\chi(t)$
of arguments that will be considered for a part of an extension; we
write~$E(\vec u)\eqdef I$ to address this extension part.
The second position consists of a
set~$O \subseteq \chi(t) \setminus I$ that represents
arguments that attack any other argument of the extension part.
Finally, the third position is the set~$D\subseteq \chi(t)$ of
arguments in the current bag that are already defeated
(counterattacked) by any argument in the extension, and therefore in a
sense compensate the set~$O$ of attacking arguments.
The idea of the algorithm is as follows. 
For nodes with~$\type(t)=\leaf$, Line~\ref{line:primleaf} initially
sets the extension part~$I$, set~$O$ of attackers, and set~$D$ of
defeated arguments to the empty set.
%f
Intuitively, in Line~\ref{line:primintr} whenever we encounter an
argument~$a$ for the first time while traversing the TD
($\type(t)=\intr$), we guess whether $a\in I$ or $a\not\in I$.
Further, we ensure that $I$ is conflict-free and that we construct
only rows where~$c\in I$ if~$a=c$. 
%
%\johannes{elab slightly before}
%
Since ultimately every argument has to be defended by the extension,
we keep track of attacking arguments in~$O$ and defeated
arguments~$D$.
In Line~\ref{line:primrem}, whenever we remove an argument~$a$
($\type(t)=\rem$), we are not allowed to store $a$ in the table any
more as the length of a row~$\vec u$ in the table~$o(t)$ depends on
the arguments that occur in the bag~$\chi(t)$; otherwise we would
exceed the length and loose the bound on the treewidth. However, we
have to ensure that either~$a$ is not an attacking argument
($a\not\in O$), or that~$a$ was defeated at some point~($a\in D$).
In the end, Condition~(ii) of a TD ensures that whenever an argument
does not occur in the bag any more, we encountered its entire
involvement in the attack relation.
Finally, Line~\ref{line:primjoin} ensures that we only combine rows
that agree on the extension and combine information concerning attacks
and defeats accordingly.  This case can be seen as a combination of
database joins ($\type(t)=\join$).
$\ADM$ can vacuously be extended to an algorithm~$\STAB$ for stable
semantics. 
There one simply drops the set~$O$ and ensures in
Line~\ref{line:primrem} that the removed atom~$a$ is either in the
extension part~$I$ or defeated (in~$a \in D$).
An algorithm~$\COMP$ for the complete semantics requires
some additional technical effort.
There one can distinguish five states, namely elements that are in the
extension, defeated ``candidates'', already defeated, candidates for
not being in the extension (unrelated), or actually proven to be
unrelated.

In the following proposition, we give more precise runtime upper
bounds for the algorithms presented in the
literature~\cite{DvorakPichlerWoltran12} that can be obtained by employing sophisticated data structures,
especially for handling nodes~$t$ with~$\type(t)=\join$.

\begin{proposition}[$\star$]
  \label{thm:runtime:basic}
  Algorithm~$\dpa_{\STAB}$ runs in
  time~$\bigO{3^{k} \cdot k \cdot g}$, $\dpa_{\ADM}$ in
  $\bigO{4^{k} \cdot k \cdot g}$, and~$\dpa_{\COMP}$ in
  $\bigO{5^{k} \cdot k \cdot g}$ where $k$ is the width and $g$ the
  number of nodes of the TD.

  %
  % Given a framework~$F$ and a TD~${\cal T} = (T,\chi)$ of~$G_F$ of
  % width~$k$ with $g$ nodes. Algorithms~$\dpa_{\STAB}$, $\dpa_{\ADM}$,
  % and~$\dpa_{\COMP}$ run in time~$\mathcal{O} (3^{k} \cdot g)$,
  % $\mathcal{O} (4^{k} \cdot g)$, and $\mathcal{O} (5^{k} \cdot g)$,
  % respectively.
  %
\end{proposition}
% \begin{proof}[Proof (Sketch)]
%   Let~$d = k+1$ be maximum bag size of the
%   TD~$\TTT$. %For each node~$t$ of $T$, we consider the
%   %table $\nu(t)$ which has been computed by
%   %$\dpa_\AlgA$. %~\cite{SamerSzeider10b}. 
%   We only discuss the case for algorithm~$\dpa_{\ADM}$ here.
%   The table~$\tau(t)$ has at most
%   $4^{d}$ rows of the form~$\langle I, \mathcal{A}, \mathcal{D}\rangle$,
%   since an argument actually can be either in one of these sets~$I,\mathcal{A},\mathcal{D}$
%   or in none of them (just modify \ADM such that~$\mathcal{A}\cap\mathcal{D}=\emptyset$).
%   In total, with the help of efficient data structures, e.g., for nodes~$t$ with~$\type(t)=\join$, one can establish a runtime bound of~$\bigO{{4^{d}}}$.
%   %since we also need to store the counters in binary representation.
%   % 
%   Then, we apply this to every node~$t$ of the TD, which resulting in
%   running
%   time~$\bigO{{4^{d}} \cdot g}\subseteq \bigO{4^{k}\cdot g}$.
%   %Consequently, the theorem holds.
%   %
% \end{proof}
%
%FichteEtAl17a,

The definitions of preferred, semi-stable, and stage semantics involve
subset-maximization. Therefore, one often introduces a concept of
witness (extension part) and counter-witness in the rows in DP, where
the counter-witness tries to invalidate subset-maximality of the
corresponding witness~\cite{JaklPichlerWoltran09}.  In the
counter-witness one stores sets of arguments that are supersets of the
considered extension, such that, in the end, at the root there was no
superset of an extension in the counter-witness while traversing the
TD.
In other words, for a witness the counter-witness failed to invalidate
maximality and accordingly the witness is subset-maximal.
In the literature, algorithms that involving such an interplay between
witnesses and counter-witnesses have been defined for preferred and
semi-stable semantics, we simply refer to them as $\dpa_\PREF$
and~$\dpa_\SEMI$. For the stage semantics, we provide the algorithm in
Listing~\ref{fig:stage}. Intuitively, we compute conflict-free
extensions during the TD traversal and additionally guess
candidates~$\mathcal{AC}$ that ultimately have to be attacked
($\mathcal{A}$) by the extension part~$J$. This allows us then to
subset-maximize upon the range part $J \cup \mathcal{AC}$, by trying
to find counter-witnesses~$\mathcal{C}$ to subset-maximality.  Again a
more detailed runtime analysis yields the following result.

\begin{proposition}[$\star$]
  \label{thm:runtime:higher}
  Algorithms~$\dpa_{\PREF}$, $\dpa_{\SEMI}$, and~$\dpa_{\STAG}$ run in
  time~$\mathcal{O} (2^{2^{4k+1}} \cdot g)$ where $k$ is the
  width  and $g$ the number of nodes of the TD.
\end{proposition}

%DP for STAGE extensions
{%
 \renewcommand{\eqdef}{\leftarrow}
 %
 % NEW VERSION
 %
 % \vspace{-5ex}
 \begin{algorithm}[t]
   \KwData{Node~$t$, bag $\chi_t$, bag-framework~$F_t=(A_t, R_t)$,
     credulous argument~$c$, and $\langle \tau_{1}, \tau_2 \rangle$ is
     the sequence of tables of~children of~$t$. \shortversion{

     }\textbf{Out: } Table~$\tau_t.$}
   %
   % \KwResult{ \ADM-table~$\tau_{t}.$}
   %
   \lIf(\hspace{-1em})
   {$\type(t) = \leaf$}{%
     $\tau_{t} \eqdef \{ \langle
     \tuplecolor{\inputPredColor}{\emptyset}, \tuplecolor{\outputPredColor}{\mathcal{\emptyset}}, \tuplecolor{\outputPredColor}{\mathcal{\emptyset}}, \tuplecolor{\specialPredColor}{\emptyset}
     \rangle \}$%\label{line:primleaf}%
     %
     % \EndIIf
   }%
   \uElseIf{$\type(t) = \intr$ and $a\hspace{-0.1em}\in\hspace{-0.1em}\chi_t$ is the introduced argum.}{
     % \vspace{-0.05em}
     % \makebox[3.9cm][l]{\hspace{-1em}
     $\hspace{-1em}\tau_{t} \eqdef %\compr(
     \{ \langle \tuplecolor{\inputPredColor}{J}, \tuplecolor{\outputPredColor}{\MAII{\mathcal{A}}{J \rightarrowtail_{R_t} A_t}}, \tuplecolor{\outputPredColor}{AC}, \tuplecolor{\specialPredColor}{\MAIII{\mathcal{C}}{\langle J, \mathcal{A}, AC \rangle}(a)} \rangle %}
     \mid\langle \tuplecolor{\inputPredColor}{I}, \tuplecolor{\outputPredColor}{\mathcal{A}}, \tuplecolor{\outputPredColor}{\mathcal{AC}}, \tuplecolor{\specialPredColor}{\mathcal{C}} \rangle\shortversion{\hspace{-5em}}$
     $\in \tau_{1}, (J, AC) \in \guess_a(I, \mathcal{AC}),
     J \cap \{c\} = \chi(t) \cap \{c\} \}\hspace{-5em}$
     % $%)
     % \hspace{-5em}$
     %\label{line:primintr}
     %
     \vspace{-0.05em}
   }\vspace{-0.05em}%
   \uElseIf{$\type(t) = \rem$ \KwAnd $a \not\in \chi_t$ is the removed argum.}{%
     % \makebox[3.9cm][l]{%\hspace{-1em}
     $\hspace{-1em}\tau_{t} \eqdef %\compr(
     \{ \langle \tuplecolor{\inputPredColor}{\MAR{I}{a}}, \tuplecolor{\outputPredColor}{\MAR{\mathcal{A}}{a}}, \tuplecolor{\outputPredColor}{\MAR{\mathcal{AC}}{a}}, \tuplecolor{\specialPredColor}{\MARR{\mathcal{C}}{a}}
     \rangle %}
     \mid\langle \tuplecolor{\inputPredColor}{I}, \tuplecolor{\outputPredColor}{\mathcal{A}}, \tuplecolor{\outputPredColor}{\mathcal{AC}}, \tuplecolor{\specialPredColor}{\mathcal{C}}
     \rangle \in \tau_{1}, a \in I \cup \mathcal{A} \}%)
     \hspace{-5em}$%\label{line:primrem}
     \vspace{-0.1em}
   } %
   \ElseIf{$\type(t) = \join$}{%, and $\Tab{} = \langle \tau', \tau'' \rangle$}{%
     % \makebox[2.9cm][l]{\hspace{-1em}
     $\hspace{-1em}\tau_{t} \eqdef %\compr(
     \{ \langle \tuplecolor{\inputPredColor}{I}, \tuplecolor{\outputPredColor}{\MAII{{\mathcal{A}_1}}{{\mathcal{A}_2}}}, \tuplecolor{\outputPredColor}{\mathcal{AC}}, \tuplecolor{\specialPredColor}{(\mathcal{C}_1 \bowtie \mathcal{C}_2) \cup}$
     $\tuplecolor{\specialPredColor}{(\mathcal{C}_1 \bowtie \{\langle u_2, \bot\rangle \}) \cup (\{\langle u, \bot\rangle \} \bowtie \mathcal{C}_2)}
     \rangle %}
     \mid u_1\!\in\!\tau_1, u_2\!\in\!\tau_2,\shortversion{\hspace{-5em}}$
     $u_1 = \langle \tuplecolor{\inputPredColor}{I}, \tuplecolor{\outputPredColor}{\mathcal{A}_1}, \tuplecolor{\outputPredColor}{\mathcal{AC}_1}, \tuplecolor{\specialPredColor}{\mathcal{C}_1} \rangle,$
     $u_2 = \langle \tuplecolor{\inputPredColor}{I}, \tuplecolor{\outputPredColor}{{\mathcal{A}_2}}, \tuplecolor{\outputPredColor}{{\mathcal{AC}_2}}, \tuplecolor{\specialPredColor}{\mathcal{C}_2}\rangle \}\hspace{-5em}$%\label{line:primjoin}
     \vspace{-0.1em}
   }
   \Return $\tau_{t}$
   \vspace{-0.25em}
   \caption{Local algorithm~$\STAG(t, \chi_t, \cdot, (F_t, c, \cdot),
     \langle \tau_1, \tau_2 \rangle)$.}
   \label{fig:stage}
   \algorithmfootnote{
     \renewcommand{\eqdef}{{\ensuremath{\,\mathrel{\mathop:}=}}}
     {
       $\guess_a(I, \mathcal{AC}) \eqdef \big\SB (J, AC) \SM$ \shortversion{\\
       \makebox[1.7cm][l]{}}$J \in \{I, \MAI{I}{a}\}, AC \in
       \{\mathcal{AC},
       \MAI{\mathcal{AC}}{a}\},$ $J \cap AC = \emptyset, $\shortversion{\\
       \makebox[1.7cm][l]{}}$[J \rightarrowtail_{R_t} J] = \emptyset,
       [A_t
       \leftarrowtail_{R_t} J] \subseteq AC\shortversion{\qquad\qquad\quad}\,\big\}$,\\[0.25em]
       $\MAIII{\mathcal{C}}{\langle J', \mathcal{A}', AC'\rangle}(a)
       \eqdef \big\SB \langle \langle {J}, {\MAII{\mathcal{A}}{J
           \rightarrowtail_{R_t} A_t}}, {AC} \rangle, {(\MAII{J}{AC}
         \subsetneq \MAII{{J'}}{{AC'}}) \vee s}\rangle \,\big|$\shortversion{\\
       \makebox[0.5cm][l]{}} $\langle \langle {I}, {\mathcal{A}},
       {\mathcal{AC}}\rangle, s\rangle \in \MAI{\mathcal{C}}{\langle
         J, \mathcal{A}, AC \rangle, \bot \rangle}, (J,AC) \in
       \guess_a(I, \mathcal{AC}),
       $\\
       \makebox[5.5cm][l]{}$J \cap \{c\} = \chi(t) \cap \{c\}
       \}$,\\[0.5em]
       $\MARR{\mathcal{C}}{a} \eqdef \SB \langle \langle \MAR{I}{a},
       \MAR{\mathcal{A}}{a}, \MAR{\mathcal{AC}}{a} \rangle, \sigma
       \rangle \SM \langle \langle I, \mathcal{A},
       \mathcal{AC}\rangle, \sigma\rangle \in \mathcal{C}, a \in I
       \cup \mathcal{A}\}$,\\[0.25em]
       $\mathcal{C}_1 \bowtie \mathcal{C}_2 \eqdef \big\SB \langle
       \langle I, \MAII{\mathcal{A}_1}{\mathcal{A}_2}, \mathcal{AC}
       \rangle, \sigma_1 \vee \sigma_2 \rangle \SM$ \shortversion{\\
       \makebox[1.8cm][l]{}}$ \langle \langle I, {\mathcal{A}_1},
       \mathcal{AC} \rangle, \sigma_1 \rangle \in \mathcal{C}_1,
       \langle \langle I, {\mathcal{A}_2}, \mathcal{AC} \rangle,
       \sigma_2 \rangle \in \mathcal{C}_2 \,\big\}$.}} %and
   % $\sigma_1 \circ \sigma_2 \eqdef \subsetneq \text{ if }\sigma_1=\subsetneq \text{ or }\sigma_2=\subsetneq; \subseteq \text{ otherwise }.}}
 \end{algorithm}%
 \renewcommand{\eqdef}{{\ensuremath{\,\mathrel{\mathop:}=}}}

}

\medskip
\noindent\textbf{Lower Bounds.}
%\paragraph{Lower Bounds.}
A natural question is whether we can
significantly improve the algorithms stated in
Propositions~\ref{thm:runtime:basic} and~\ref{thm:runtime:higher}. In
other words, we are interested in lower bounds on the runtime of an
algorithm that exploits treewidth for credulous reasoning. A common
method in complexity theory is to assume that the \emph{exponential
  time hypothesis (ETH)} holds and establish reductions. The ETH
states that there is some real~$s > 0$ such that we cannot decide
satisfiability of a given 3-CNF formula~$\phi$ in
time~$2^{s\cdot\Card{\phi}}\cdot\CCard{\phi}^{\mathcal{O}(1)}$~\cite[Ch.14]{CyganEtAl15}.
Subsequently, we establish new results assuming ETH employing known
reductions from the literature.
Essentially, our lower bounds show that there is no hope for a better
algorithm.

\begin{theorem}[$\star$]\label{prop:eth_asp}
  Let $\SEM \in \{$admissible, complete, stable$\}$, $F$ be a framework and
  $k$ the treewidth of the underlying graph~$G_F$. Unless ETH fails,
  $\Cred_\SEM$ \emph{cannot} be solved in
  time~$2^{o(k)} \cdot \CCard{F}^{o(k)}$ and for
  $\SEM = \text{semi-stable}$, $\Cred_{\SEM}$ and $\Skep_\SEM$
  \emph{cannot} be solved in
  time~$2^{2^{o(k)}} \cdot \CCard{F}^{o(k)}$.
\end{theorem}%
\begin{proof}[Proof (Idea)]
  The existing reductions by~\citex{DunneBench-Capon02a} increase the
  treewidth only linearly and are hence sufficient. For semi-stable
  statements the reductions by~\citex{DvorakWoltran10} can
  be applied, since preferred and semi-stable extensions of the
  constructed argumentation framework coincide.
\end{proof}

\newcommand{\llangle}{\ensuremath{\langle\hspace{-2pt}\{\hspace{-0.2pt}}}
\newcommand{\rrangle}{\ensuremath{\}\hspace{-2pt}\rangle}}
%fixme!
\newcommand{\STab}{\ensuremath{\ATab{\AlgA}}}%
\section{Algorithms for Projected Credulous Counting by exploiting \shortversion{Bounded }Treewidth}
In the previous section, we presented algorithms to solve the
reasoning problems for abstract argumentation. These algorithms can be
extended relatively straightforward to also count extensions without
projection by adding counters to each row at a quadradtic runtime
instead of linear in the size of the input instance. One can even
reconstruct extensions~\cite{PichlerRuemmeleWoltran10}.
However, things are more complicated for projected credulous counting.

In this section, we present an algorithm~$\mdpa{\mathbb{S}}$ that
solves the projected credulous counting problem ($\PCC_\SEM$) for
semantics~$\SEM \in \ALL$. 
Our algorithm lifts results for projected model counting in the
computationally and conceptually much easier setting of propositional
satisfiability~\cite{FichteEtAl18} to abstract argumentation.
Our algorithm is based on dynamic programming and traverses a TD three
times. To this end, we employ algorithms~$\mathbb{S} \in \{\ADM,$
\COMP, \PREF, \STAG, \SEMI, $\STAB\}$ as presented in the
previous section according to the considered semantics~$\SEM$.
The first traversal consists of~$\dpa_{\mathbb{S}}$,
where~$\mathbb{S}$ is a local algorithm for credulous reasoning of the
chosen semantics, which results in
TTD~$\mathcal{T}_{\mathbb{S}\hy\textit{Cred}}=(T,\chi,\tau)$.

In the following, let again $F=(A,R)$ be the given framework, $a\in A$ an
argument, $(T,\chi)$ a TD of~$G_F$ with $T=(N,\cdot,n)$ and the
root~$n$, and $\TTT_{\mathbb{S}\hy\textit{Cred}}=(T, \chi, \tau)$ be
the TTD that has been computed by the respective algorithms as
described in the previous section.
Then, we intermediately
traverse~$\mathcal{T}_{\mathbb{S}\hy\textit{Cred}}$ in pre-order and
\emph{prune irrelevant rows}, thereby we remove all rows that cannot
be extended to a credulous extension of the corresponding
semantics~$\SEM$. We call the resulting TTD
$\mathcal{T}_{\mathbb{S}\hy\textit{Pruned}}=(T,\chi, \nu)$.  Note that
pruning does not affect correctness, as only rows are removed where
already the count without even considering projection is~$0$.
However, pruning serves as a technical trick for the last traversal to
avoid corner cases, which result in correcting counters and
backtracking.

%
%
%
% %\begin{definition}[c.f.,\citey{FichteEtAl18}]\label{def:origin}
% Let therefore $F$ be a framework, $\TTT=(T, \chi, \tau)$ be an~$\AlgA$-TTD
% of~$G_F$, and $t$ be a node of~$T$ where
% $\children(t,T)=\langle t_1, \ldots, t_{\ell}\rangle$. %, and
% % 
% % 
% %, for technical reasons.
% % 
% % 
% For a given $\AlgA$-row~$\vec u$, we define the originating\footnote{For sequence~$\vec s=\langle s_1, \ldots, s_{\ell} \rangle$, let
% $\llangle \vec s\rrangle \eqdef \langle \{s_1\}, \ldots, \{s_{\ell}\}
% \rangle$.}
% $\AlgA$-rows of~$\vec u$ in node~$t$ by
% % 
% % \begin{align*}
% $\orig(t,\vec u) \eqdef \SB \vec s \SM \vec s \in \tau(t_1)
% \times \cdots \times \tau({t_\ell}), \vec u \in {\AlgA}(t,\chi(t),
% \cdot,(F_t,\cdot), \llangle \vec s\rrangle) \SE.$ %
% % \end{align*}
% We extend this to an $\AlgA$-table~$\sigma$ by
% $\origs(t,\sigma) \eqdef \bigcup_{\vec u \in \sigma}\orig(t,\vec u)$.

%
%
In the final traversal, we count the projected credulous extensions.
Therefore, we compute a
TTD~$\mathcal{T}_{\mathbb{S}\hy\textit{Proj}}=(T,\chi,\pi)$ using
algorithm~$\dpa_{\PROJ}$ using local algorithm~$\PROJ$ as given in
Listing~\ref{fig:dpontd3}.  Algorithm~$\PROJ$ stores for each node a
pair~$\langle \sigma, c\rangle \in \pi(t)$,
where~$\sigma\subseteq \nu(t)$ is a table~$\nu(t)$ from the previous
traversal and~$c\geq 0$ is an integer representing what we call the
\emph{intersection projected count ($\ipmc$)}.

Before we start with explaining how to obtain these $\ipmc$
values~$c$, we require auxiliary notations from the literature. First,
we require a notion to reconstruct extensions from~$\TTT$, more
precisely, for a given row to define its predecessor rows in the
corresponding child tables.
Therefore, let $t$ be a node of~$T$ with children~$t_1$ and $t_2$, if
it exists. Since sequences used in the following depend on number of
the children assume for simplicity of the presentation that sequences
are implicitly of corresponding length even if they are given as of
length~2.  For a given row~$\vec u \in \tau(t)$, we define the
originating%
\footnote{%
  For sequence~$\vec s=\langle s_1, s_2 \rangle$, %of length $\leq 2$,
  let
  $\llangle \vec s\rrangle \eqdef \langle \{s_1\}, \{s_2\} \rangle$.
}
 rows of~$\vec u$ in node~$t$ by
$\orig(t,\vec u) \eqdef \SB \vec s \SM \vec s \in \tau(t_1) \times
\tau({t_2}), \vec u \in {\mathbb{S}}(t,\chi(t), \cdot,(F_t,\cdot),
\llangle \vec s\rrangle) \SE$ and for a table~$\sigma$ as the union
over the origins for all rows~$\vec u \in \sigma$.
% similar for a table~$\sigma$ by
% $\origs(t,\sigma) \eqdef \bigcup_{\vec u \in \sigma}\orig(t,\vec u)$.
%
% \begin{example}
% \end{example}
%
Next, let~$\sigma\subseteq\nu(t)$.  In order to combine rows and solve
projection accordingly, we need equivalence classes of rows.  Let
therefore relation~$\bucket \subseteq \sigma \times \sigma$ consider
equivalent rows with respect to the projection of its extension part
by %as follows %
$\bucket \eqdef \SB (\vec u,\vec v) \SM \vec u, \vec v \in \sigma,
{{E}(\vec u)} \cap {P} = {{E}(\vec v)} \cap {P}\SE.$
Let $\buckets_P(\sigma)$ be the set of equivalence classes induced
by~$\bucket$ on~$\sigma$,~i.e.,
$\buckets_P(\sigma) \eqdef\, (\sigma / \bucket) = \SB [\vec u]_P \SM
\vec u \in \sigma\SE$, where
$[\vec u]_P = \SB \vec v \SM \vec v \bucket \vec u,\vec v \in
\sigma\}$~\cite{Wilder12a}.
% 
%  Further, 
%  %we define the set of all %$\subbuckets_P(\tau)$ of all
%  %sub-equivalence classes of~$\tau$ by
%  $\subbuckets_P(\rho) \eqdef \SB S \SM \emptyset \subsetneq S
%  \subseteq B, B \in \buckets_P(\rho)\SE$ covers all non-empty subsets of
%  equivalence classes.

%\begin{example}
%\end{example}

When computing the $\ipmc$ values~$c$ stored in each row~$\vec u$
of~$\pi(t)$, we compute a so-called \emph{projected count ($\pmc$)} as
follows.  
First, we define the \emph{stored $\ipmc$}
of~$\sigma \subseteq \nu(t)$ in table~$\pi(t)$ by
$\sipmc(\pi(t), \sigma) \eqdef \sum_{\langle \sigma, c\rangle \in
  \pi(t)} c.$ We use the $\ipmc$ value in the context of ``accessing''
$\ipmc$ values in table~$\pi(t_i)$ for a child~$t_i$ of~$t$.
This can be generalized to a
sequence~$s=\langle \pi(t_1), \pi(t_2)\rangle$ of tables and a
set~$O = \{\langle \sigma_1, \sigma_2\rangle, \langle \sigma_1',
\sigma_2'\rangle, \ldots\}$ of sequences of tables by
$\sipmc(s, O)=\sipmc(s_{(1)},O_{(1)})\cdot \sipmc(s_{(2)},O_{(2)})$.
%
%This allows to select the $i$-th position of the sequence together
%with sets of the $i$-th positions from the set of sequences.
%
%Intuitively, when we are at a node~$t$ in algorithm~$\dpa_\PROJ$ we
%have already computed~$\pi(t')$ of $\TTT_{\text{proj}}$ for every node~$t'$
%below~$t$.
%
%Then, we compute the projected answer set count
%of~$\rho \subseteq \nu(t)$. Therefore, we apply the
%inclusion-exclusion principle to the stored projected answer set count
%of origins.
%
Then, the \emph{projected count~$\pmc$} of
rows~$\sigma\subseteq\nu(t)$ is the application of the
inclusion-exclusion principle to the stored intersection projected
counts,~i.e., $\ipmc$ values of children of~$t$.  Therefore, $\pmc$
determines the origins of table~$\sigma$, and uses the stored counts
($\sipmc$) in the \PROJ-tables of the children~$t_i$ of~$t$ for all
subsets of these origins.
Formally, we define\shortversion{\smallskip\\
$\pmc(t,\sigma, \langle\pi(t_1),\pi(t_2)\rangle) \eqdef$\\
%\hspace{2em}
\mbox{~~~~~~}$\sum_{\emptyset \subsetneq O \subseteq
  {\origs(t,\sigma)}} (-1)^{(\Card{O} - 1)} \cdot \sipmc(\langle
\pi(t_1), \pi(t_2)\rangle, O)$.\smallskip\\}
\longversion{
\[\pmc(t,\sigma, \langle\pi(t_1),\pi(t_2)\rangle) \eqdef
\sum_{\emptyset \subsetneq O \subseteq
  {\origs(t,\sigma)}} (-1)^{(\Card{O} - 1)} \cdot \sipmc(\langle
\pi(t_1), \pi(t_2)\rangle, O).\]}
\noindent Intuitively, $\pmc$ defines the number of distinct projected
extensions in framework~$F_{\leqslant t}$ to which any row in~$\sigma$
can be extended.
%\begin{example}
%\end{example}
%
Finally, the \emph{intersection projected count}~$\ipmc$ for~$\sigma$
is the result of another application of the inclusion-exclusion
principle. It describes the number of common projected $\SEM$\hy
extensions which the rows in~$\sigma$ have in common in
framework~$F_{\leqslant t}$.  We define 
$\ipmc(t,\sigma,s)\eqdef 1$ if
$\type(t) = \leaf$ and otherwise
$\ipmc(t,\sigma,s)\eqdef$
$ \big|\pmc(t,\sigma, s)$ $+
\sum_{\emptyset\subsetneq\varphi\subsetneq\sigma}(-1)^{\Card{\varphi}}
\cdot \ipmc(t,\varphi, s)\big|$,
where $s = \langle \pi(t_1), \pi(t_2)\rangle$.
In other words, if a node is of type~$\leaf$, $\ipmc$ is one, since
bags of leaf nodes are empty. Observe that since bags~$\chi(n)$ for
root node~$n$ are empty, there is only one entry in~$\pi(n)$
and~$\pmc(n,\nu(n),s)=\ipmc(n,\nu(n), s)$, %(for~$s$ as above)
which corresponds to the number of projected credulous extensions.
In the end, we collect~$\pmc$\hy values for all subsets of~$\nu(t)$.

\begin{algorithm}[t]
  \KwData{%
    Node~$t$, table~$\nu_{t}$ after purging, set~$P$ of projection
    atoms, $\langle \pi_1, \pi_2 \rangle$ is the sequence of tables at
    the children of~$t$.
    % LTD~$\mathcal{T}=(T,\cdot,\langle \tau, \pi \rangle)$, Node~$t$,
    % set~$P$ of projection atoms, $\Tab{}$.
    % \hspace{2em}
    % $\ASol{\AlgS}$: maps~$t$ to
    % $\AlgS$-table. %obtained by~$\AlgS$.\hspace{-5em}
    %
  }%
  \KwResult{Table~$\pi_{t}$ of pairs~$\langle \sigma, c\rangle$,
    where $\sigma \subseteq \nu_{t}$, and $c \in \Nat$.\hspace{-5em}
  } %
  $\makebox[0em]{}\pi_{t}\hspace{-0.2em}\leftarrow\hspace{-0.2em}\big\SB\langle
  \sigma, \ipmc(t,\sigma,\langle \pi_{1}, \pi_2\rangle) \rangle
  \big{|}\, \emptyset \subsetneq \sigma \subseteq
  \buckets_P(\nu_{t})\,\big\}\hspace{-5em}$ \;
  \Return{$\pi_{t}$}
  \vspace{-0.15em}
  \caption{Local algorithm
    $\PROJ(t, \cdot, \nu_t, (\cdot, \cdot, P), \langle \pi_1, \pi_2 \rangle)$
    for projected counting,~c.f.,~\cite{FichteEtAl18}.}
  \label{fig:dpontd3}
\end{algorithm}

%\begin{example}
%\end{example}

\begin{theorem}[$\star$]
  Algorithm~$\mdpa{\mathbb{S}}$ is correct and solves $\PCC_{\SEM}$
  for local algorithms~$\mathbb{S}\in\{\ADM$,
  $\COMP$, $\PREF, \STAG, \SEMI, \STAB\}$, i.e.,
  $\sipmc(\pi(n),\emptyset)$ returns the projected credulous count at
  the root~$n$ for corresponding semantics~$\SEM$.
\end{theorem}
\begin{proof}[Proof (Idea)]
  In order to prove correctness, we can establish an invariant for
  each row of each table. Then, we show this invariant by simultaneous
  structural induction on pc and ipc starting at the leaf nodes and
  stepping until the root. This yields that the intersection projected
  count for the empty root corresponds to~$\PCC_\SEM$ for the
  semantics~$\SEM$.  For completeness, we demonstrate by induction
  from root to leaves that a well-defined row of one table, which can
  indeed be obtained by the corresponding table algorithm, always has
  some preceding row in the respective child nodes.
\end{proof}

\medskip
\noindent\textbf{Runtime Bounds (Upper and Lower).}
In the following, we present upper bounds on algorithm~$\PROJ$ that
immediately result in runtime results for~$\mdpa{\mathbb{S}}$.  Let
therefore~$\gamma(n)$ be the number of operations required to multiply
two~$n$-bit integers.  Note that
$\gamma(n)\in O(n\cdot \log(n) \cdot \log(\log(n)))$~\cite{Knuth1998}.

\begin{proposition}[$\star$, \citey{FichteHecher18b}]%[$\star$]
  \label{thm:runtime}
  $\dpa_{\PROJ}$ runs in time
  $\bigO{2^{4m}\cdot g \cdot \gamma(\CCard{F})}$\footnote{The
    value~$m$ depends on the treewidth~$k$. However, the actual order
    depends on the semantics.} %
  where $g$ is the number of nodes of the given TD of the underlying
  graph~$G_F$ of the considered framework~$F$ and
  $m\eqdef \max\{\Card{\nu(t)} \mid t\in N\}$ for input
  TTD~$\TTT_{\text{purged}} = (T,\chi,\nu)$ of $\dpa_\PROJ$.
\end{proposition}%

\begin{corollary}
  For~$\mathbb{S} \in \{\ADM$, \COMP, $\STAB\}$, $\mdpa{\mathbb{S}}$
  runs in time~$\bigO{2^{2^{4k}}\cdot g \cdot
    \gamma(\CCard{F})}$. For~$\mathbb{S}\in\{\PREF$, \SEMI, $\STAG\}$,
  runs in time~$\bigO{2^{2^{2^{4k}}}\cdot g \cdot \gamma(\CCard{F})}$
  where $k$ is the treewidth of the underlying graph~$G_F$ of the
  given AF~$F$.
\end{corollary}

% From the theorem we immediately obtain that
% runs in linear time while being double exponential in the treewidth, whereas $\mdpa{\AlgA}$ for~$\AlgA\in\{\PREF, \SEMI, \STAG\}$
% requires triple exponential time (in the treewidth).

Next, we take again the exponential time hypothesis (ETH) into account
to establish lower bounds for counting projected extensions.  In
particular, we obtain that under reasonable assumptions, we cannot
expect to improve the presented algorithms significantly.

\begin{theorem}[$\star$]\label{prop:eth_proj_counting}
  Let $\SEM \in\{$admissible, complete, stable$\}$. Unless ETH fails,
  we cannot solve the problem~$\PCC_\SEM$ in
  time~$2^{2^{o(k)}} \cdot \CCard{F}^{o(k)}$ where~$k$ is the
  treewidth of the underlying graph~$G_F$ of the considered framework~$F$.
\end{theorem}%
\begin{proof}[Proof (Sketch)]
  We establish the lower bound by reducing an instance
  of~$\forall\exists\hy$\SAT to an instance of a version of
  $\Cred_\SEM$ where the extension is of size exactly~$\ell$.
  %
  % the decision version of projected credulous counting, namely
  % \emph{projected credulous count exactly~$\ell$}.
  Note that under ETH the problem $\forall\exists\hy$\SAT cannot be
  solved~\cite{LampisMitsou17} in
  time~$2^{2^{o(k)}} \cdot \CCard{F}^{o(k)}$ in the worst case.
  We follow the reduction from the proof of Statement~2 in
  Lemma~\ref{lem:pcct-numberdotcnp-hard}. Let~$\ell=\Card{X}$, and
  observe that we can compute reduction in polynomial-time and the
  treewidth of the projected credulous counting instance is increased
  only linearly. It is easy to see that the reduction is correct
  since~$\Card{B(AF,X,t)}=\ell=\Card{X}$ if and only
  if~$\varphi(X)=\exists Y \psi(X,Y)$ holds for all assignments
  using~$X$. Consequently, the claim follows.
\end{proof}

For semi-stable, preferred and stage semantics, we believe that this
lower bound is not tight. Hence, we apply a stronger version (3ETH) of
the ETH for quantified Boolean formulas (QBF). However, it is open
whether also ETH implies 3ETH.

\begin{hypothesis}[3ETH, \citey{FichteHecher18b}]\label{hyp:lampis3}
  The problem $\exists\forall\exists$-\SAT for a quantified Boolean
  formula~$\Phi$ of treewidth~$k$ can not be decided in
  time~${2^{2^{2^{o(k)}}}}\cdot \CCard{\Phi}^{o(k)}$.
\end{hypothesis}

Using this hypothesis, we establish the following result.

\begin{theorem}[$\star$]\label{thm:lowerbound_semi_pref}
  Let $\SEM \in \{$preferred, semi-stable, stage semantics$\}$. Unless
  3ETH fails, we cannot solve the problem~$\PCC_\SEM$ in
  time~$2^{2^{2^{o(k)}}} \cdot \CCard{F}^{o(k)}$ where~$k$ is the
  treewidth of the underlying graph of~$F$.
\end{theorem}
\begin{proof}[Proof (Idea)]
  Assuming Hypothesis~\ref{hyp:lampis3} we cannot solve an instance
  of~$\forall\exists\forall$-\SAT in
  time~$2^{2^{2^{o(k)}}} \cdot \CCard{F}^{o(k)}$, otherwise we could
  solve an instance~$\Phi$ of~$\exists\forall\exists$-\SAT, using a
  decision procedure for~$\forall\exists\forall$-\SAT with the inverse
  of~$\Phi$ and inverting the result, in
  time~$2^{2^{2^{o(k)}}} \cdot \CCard{F}^{o(k)}$.
  Towards the lower bound, we finally establish a reduction
  from~$\forall\exists\forall$-\SAT to projected credulous count
  exactly~$\ell$ (c.f.,
  Theorem~\ref{prop:eth_proj_counting}). Thereby, we apply the
  reduction provided in Statement~1 of
  Lemma~\ref{lem:pcct-numberdotcnp-hard}, set~$\ell\eqdef \Card{X}$
  and proceed analogously to Theorem~\ref{prop:eth_proj_counting}.
\end{proof}

\section{Conclusion and Outlook}
We established the classical complexity of counting problems in
abstract argumentation. We complete these results by presenting an
algorithm that solves counting projected credulous extensions when
exploiting treewidth in runtime double exponential in the treewidth or
triple exponential in the treewidth depending on the considered
semantics.
Further, assuming ETH or a version for 3QBF, we establish that the
runtime of the algorithms are asymptotically tight and we cannot
significantly improve on the runtime for algorithms that exploit
treewidth.
While the upper bounds in Lemma~\ref{lem:numberext-in-numberp} can be
easily transferred to counting the number of extensions of a specific
kind, the corresponding lower bounds cannot be immediately adopted
from Lemma~\ref{lem:numberext-numberp-hard}.
An open question is to investigate whether $\numberDotCoNP$-hardness
also applies for the preferred semantics.
An interesting further research direction is to study whether we can
obtain better runtime results by designing algorithms that take in
addition also the number (small or large) of projection arguments into
account or
to study whether an implementation of our approach can benefit from
massive parallelization~\cite{FichteEtAl18a}.
Finally, our technique might also be applicable to problems such as 
circumscription~\cite{DurandHermannKolaitis05}, default
logic~\cite{FichteHecherSchindler18a}, or
QBFs~\cite{CharwatWoltran16a}.
Considering the (parameterized) enumeration complexity
\cite{JohnsonPY88,cmmsv17,ckmmov15} of the studied problems is also
planned as future work.
%
% ADD IN CAMERA READY
%

 % \bibliographystyle{aaai}
 % \bibliography{argu_counting_aaai19_refs}
 \cleardoublepage
%\end{document}

\shortversion{
	
}
\longversion{
	\bibliographystyle{abbrvnat}
	\bibliography{argu_counting_aaai19_refs.bib}
}
\longversion{%
%\subparagraph*{Acknowledgements}

\newpage

%% Bibliography
%
%\printbibliography

%\bibliography{incFPT}

%\end{document}
\clearpage
\appendix

% \johannes{STORE FOR POTENTIAL USEFUL STUFF}
% \johannes{Needed: put in place}
% Similar as for sequences when addressing the $i$-th element, for a
% set~$U$ of rows (table) we let
% $U_{(i)}\eqdef\{\vec u_{(i)} \mid \vec u \in U\}$.

% We call a decomposition that has been computed by a traversal that ran
% $\AlgA$ at each node a $\AlgA$-TTD of the given input instance.

% \optional{%
%   For a row~$\vec u$, we refer by $\vec y_{(i)}$ to the $i$-th element
%   for a given positive integer~$i$.
% %
% }

% Assuming that the exponential time hypothesis (ETH) holds, which
% states that there is some real~$s > 0$ such that we cannot decide
% satisfiability of a given 3-CNF formula~$F$ in
% time~$2^{s\cdot\Card{F}}\cdot\CCard{F}^{\mathcal{O}(1)}$, we show that
% one \emph{cannot} count projected extensions in time doubly
% exponential in the treewidth. \johannes{reformulate a bit, add all the
%   results, maybe use greek letters instead of framework~$F$}

% \johannes{STORE END}

\section{Additional Resources}

\begin{figure}
	\centering
	\begin{tikzpicture}[every node/.style={circle,draw=black,thin}]
		%\node (u) {$u$};
		\node (v) [] {$v$};
		\node (w) [right=2em of v] {$w$};
		\node (x) [right=2em of w] {$x$};
		\node (y) [right=2em of x] {$y$};
		\node (z) [right=2em of y] {$z$};
		\draw [stealth'-stealth'] (w) -- (x);
		\draw [-stealth'] (w) to [bend right=45] (y);
		%\draw [->] (x) -- (y);
		\draw [-stealth'] (z) to [bend right=45] (x);
		%\draw [->] (x) to [bend right=45] (w);
		\path[-stealth'] (z) edge [loop above, >=stealth'] ();
		%\edge [loop below] (w);
	\end{tikzpicture}
	\caption{Argumentation framework $F_{\mathit{Ex}}$.}
	\label{fig:argu}
\end{figure}
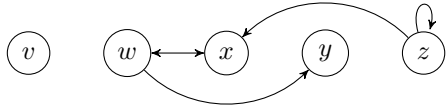

\begin{example}
  Figure~\ref{fig:argu} illustrates an AF where $F_1=(A_1, R_1)$ with
  $A_1=\{v,w,x,y,z\}$ and $R_1=\{(w,x),(x,w),$ $(w,y),(z,z),(z,x)\}$,
  c.f.~\cite[Ex.~2.7]{BliemHecherWoltran16a}. %~\cite{Hecher15a}.
  Observe that $\emptyset\in\conflictfree({F_1})$. For every
  $a\in A_1$ such that $a\not=z$ it holds that
  $\{a\}\in\conflictfree({F_1})$; since $v$ is isolated, also
  $\{v,a\}\in\conflictfree({F_1})$ for every $a\in A_{F_1}$ with
  $a\not=z$. Argument $z$ is not contained in any
  $S\in \conflictfree({F_1})$, since it attacks itself.  Finally,
  $\conflictfree({F_1})=\{\emptyset, \{v\},$
  $\{w\},\{x\},\{y\},\{v,w\},$ $\{v,x\},\{v,y\},\{x,y\},$
  $\{v,x,y\}\}$.
  Argument $x$ can never be part of any admissible extension as $z$
  has a self-loop. %, therefore also $y$ is not in any extension.
  We have that
  $\admissible({F_1})=\{\emptyset, \{v\}, \{w\}, \{v,w\}\}$.
  The set $\emptyset$ is not complete since
  $\adef_{F_1}(\emptyset)=\{v\}$; $\{w\}\not\in\complete(\{w\})$,
  since $\adef_{F_1}(\{w\})=\{v,w\}$.  In the end,
  $\complete({F_1})=\{\{v\}, \{v,w\}\}$.
  Observe that $\preferred({F_1})=\semistable(F_1)=\stage(F_1)=\{\{v,w\}\}$.
  Finally, since $z$ is not contained in any extension
  $S\in\conflictfree({F_1})$ and it is not attacked by any $a\in S$
  ($z$ only attacks itself), there cannot be any stable extension.
\end{example}

\subsection{Additional Table Algorithms}

Listing~\ref{fig:confl} presents a local algorithm~$\CONFL$ for conflict-free extensions,
whose core is also used in Listing~\ref{fig:stage}. A local algorithm~$\STAB$ for stable extensions,
which, in fact, is a simplification of Listing~\ref{fig:prim}, is provided in Listing~\ref{fig:stabl}.
Finally, Listing~\ref{fig:compl} depicts an algorithm~$\COMP$ for complete semantics working
with five different states, as mentioned in Section~``DP for Abstract Argumentation''.
For computing preferred semantics via dynamic programming ($\dpa_\PREF$),
one can use the idea of the local algorithm~$\ADM$ for admissible semantics
and subset-maximize using counterwitnesses (similar to Listing~\ref{fig:stage}) accordingly.
Finally, local algorithm~$\SEMI$ finally is similar to~$\STAB$, but relies on the idea of~$\ADM$.

%DP for CONF-FREE extensions
{
	\renewcommand{\eqdef}{\leftarrow}
	%
	% NEW VERSION
	%
	%\vspace{-5ex}
	 \begin{algorithm}[h]
	   \KwData{Node~$t$, bag $\chi_t$, bag-framework~$F_t=(A_t, R_t)$,
	   credulous argument~$c$, and
	     $\langle \tau_{1}, \tau_2 \rangle$ is the sequence of tables
	     of~children of~$t$. \shortversion{

	     }\textbf{Out: } Table~$\tau_t.$}
	     %
	     %\KwResult{ \ADM-table~$\tau_{t}.$}
	   %
	   \lIf(\hspace{-1em})
	   {$\type(t) = \leaf$}{%
	     $\tau_{t} \eqdef \{ \langle
	     \tuplecolor{\inputPredColor}{\emptyset}
	     \rangle \}$%\label{line:primleaf}%
	     %
		%\EndIIf
	   }%
	  \uElseIf{$\type(t) = \intr$ and $a\hspace{-0.1em}\in\hspace{-0.1em}\chi_t$ is the introduced argum.}{
	   %\vspace{-0.05em}
	   %\makebox[3.9cm][l]{\hspace{-1em}
	   $\tau_{t} \eqdef %\compr(
	   \{ \langle \tuplecolor{\inputPredColor}{J} %\tuplecolor{\outputPredColor}%{\MAII{\mathcal{A}}{A_t \rightarrowtail_{R_t} J}},
	    \rangle %}
	     \mid\langle \tuplecolor{\inputPredColor}{I} %\tuplecolor{\outputPredColor}{\mathcal{A}},
\rangle\in \tau_{1}, J \in \{I, \MAI{I}{a}\},$
	    % \makebox[2.9cm][l]{}
	    $J \rightarrowtail_{R_t} J = \emptyset, J \cap \{c\} = \chi(t) \cap \{c\} \}$
	      %$%)
	      %\hspace{-5em}$
	      %\label{line:primintr}
	     %
	   \vspace{-0.05em}
	     }\vspace{-0.05em}%
	     \uElseIf{$\type(t) = \rem$ \KwAnd $a \not\in \chi_t$ is the removed argum.}{%
%	       \makebox[3.9cm][l]{%\hspace{-1em}
	       $\tau_{t} \eqdef %\compr(
	       \{ \langle \tuplecolor{\inputPredColor}{\MAR{I}{a}}%, %\tuplecolor{\outputPredColor}{\MAR{\mathcal{A}}{a}},
	        %\tuplecolor{\statePredColor}{\MAR{\mathcal{D}}{a}}
	       \rangle %}
	       \mid\langle \tuplecolor{\inputPredColor}{I}%,
	       %\tuplecolor{\statePredColor}{\mathcal{D}}
	       \rangle \in \tau_{1} \}%)
	       \hspace{-5em}$%\label{line:primrem}
	       \vspace{-0.1em}
	     } %
	     \ElseIf{$\type(t) = \join$}{%, and $\Tab{} = \langle \tau', \tau'' \rangle$}{%
	       %\makebox[2.9cm][l]{\hspace{-1em}
	       $\tau_{t} \eqdef %\compr(
	       \{ \langle \tuplecolor{\inputPredColor}{I}%, %\tuplecolor{\outputPredColor}{\MAII{{\mathcal{A}_1}}{{\mathcal{A}_2}}},
	        %\tuplecolor{\statePredColor}{\MAII{{\mathcal{D}_1}}{{\mathcal{D}_2}}}
		 \rangle %}
		 \mid\langle \tuplecolor{\inputPredColor}{I}%, %\tuplecolor{\outputPredColor}{\mathcal{A}_1},
		 \rangle\in \tau_{1},$
		 $\langle \tuplecolor{\inputPredColor}{I}%,
		 %\tuplecolor{\outputPredColor}{{\mathcal{A}_2}},
		 \rangle\in \tau_{2}\}\hspace{-5em}$%\label{line:primjoin}
	       \vspace{-0.1em}
	     }
	     \Return $\tau_{t}$
	     \vspace{-0.25em}
	     \caption{Local algorithm~$\CONFL(t, \chi_t, \cdot, (F_t, c, \cdot),
	       \langle \tau_1, \tau_2 \rangle)$.}
	 \label{fig:confl}%\algorithmfootnote{
	%\renewcommand{\eqdef}{{\ensuremath{\,\mathrel{\mathop:}=}}}
	 % $\MAII{S}{S'} \eqdef S \cup S'$,
	 % $\MAI{S}{e} \eqdef S \cup \{e\}$, and
	 % $\MAR{S}{e} \eqdef S \setminus \{e\}$.}
	\end{algorithm}%
	\renewcommand{\eqdef}{{\ensuremath{\,\mathrel{\mathop:}=}}}

}

%DP for STABLE extensions
{
	\renewcommand{\eqdef}{\leftarrow}
	%
	% NEW VERSION
	%
	%\vspace{-5ex}
	 \begin{algorithm}[h]
	   \KwData{Node~$t$, bag $\chi_t$, bag-framework~$F_t=(A_t, R_t)$,
	   credulous argument~$c$, and
	     $\langle \tau_{1}, \tau_2 \rangle$ is the sequence of tables
	     of~children of~$t$. \shortversion{

	     }\textbf{Out: } Table~$\tau_t.$}
	     %
	     %\KwResult{ \ADM-table~$\tau_{t}.$}
	   %
	   \lIf(\hspace{-1em})
	   {$\type(t) = \leaf$}{%
	     $\tau_{t} \eqdef \{ \langle
	     \tuplecolor{\inputPredColor}{\emptyset}, \tuplecolor{\statePredColor}{\emptyset}
	     \rangle \}$%\label{line:primleaf}%
	     %
		%\EndIIf
	   }%
	  \uElseIf{$\type(t) = \intr$ and $a\hspace{-0.1em}\in\hspace{-0.1em}\chi_t$ is the introduced argum.}{
	   %\vspace{-0.05em}
	   %\makebox[3.9cm][l]{\hspace{-1em}
	   $\tau_{t} \eqdef %\compr(
	   \{ \langle \tuplecolor{\inputPredColor}{J}, %\tuplecolor{\outputPredColor}%{\MAII{\mathcal{A}}{A_t \rightarrowtail_{R_t} J}},
	    \tuplecolor{\statePredColor}{\MAII{{D}}{J \leftarrowtail_{R_t} A_t}} \rangle %}
	     \mid\langle \tuplecolor{\inputPredColor}{I}, %\tuplecolor{\outputPredColor}{\mathcal{A}},
\tuplecolor{\statePredColor}{{D}} \rangle\in \tau_{1}, J \in \{I, \MAI{I}{a}\},$
	    % \makebox[2.9cm][l]{}
	    $J \rightarrowtail_{R_t} J = \emptyset, J \cap \{c\} = \chi(t) \cap \{c\} \}$
	      %$%)
	      %\hspace{-5em}$
	      %\label{line:primintr}
	     %
	   \vspace{-0.05em}
	     }\vspace{-0.05em}%
	     \uElseIf{$\type(t) = \rem$ \KwAnd $a \not\in \chi_t$ is the removed argum.}{%
%	       \makebox[3.9cm][l]{%\hspace{-1em}
	       $\tau_{t} \eqdef %\compr(
	       \{ \langle \tuplecolor{\inputPredColor}{\MAR{I}{a}}, %\tuplecolor{\outputPredColor}{\MAR{\mathcal{A}}{a}},
	        \tuplecolor{\statePredColor}{\MAR{{D}}{a}}
	       \rangle %}
	       \mid\langle \tuplecolor{\inputPredColor}{I},
	       \tuplecolor{\statePredColor}{{D}}
	       \rangle \in \tau_{1}, a \in I \cup {D} \}%)
	       \hspace{-5em}$%\label{line:primrem}
	       \vspace{-0.1em}
	     } %
	     \ElseIf{$\type(t) = \join$}{%, and $\Tab{} = \langle \tau', \tau'' \rangle$}{%
	       %\makebox[2.9cm][l]{\hspace{-1em}
	       $\tau_{t} \eqdef %\compr(
	       \{ \langle \tuplecolor{\inputPredColor}{I}, %\tuplecolor{\outputPredColor}{\MAII{{\mathcal{A}_1}}{{\mathcal{A}_2}}},
	        \tuplecolor{\statePredColor}{\MAII{{{D}_1}}{{{D}_2}}}
		 \rangle %}
		 \mid\langle \tuplecolor{\inputPredColor}{I}, %\tuplecolor{\outputPredColor}{\mathcal{A}_1},
		 \tuplecolor{\statePredColor}{{D}_1} \rangle\in \tau_{1},$
		 $\langle \tuplecolor{\inputPredColor}{I},
		 %\tuplecolor{\outputPredColor}{{\mathcal{A}_2}},
		 \tuplecolor{\statePredColor}{{{D}_2}} \rangle\in \tau_{2}\}\hspace{-5em}$%\label{line:primjoin}
	       \vspace{-0.1em}
	     }
	     \Return $\tau_{t}$
	     \vspace{-0.25em}
	     \caption{Local algorithm~$\STAB(t, \chi_t, \cdot, (F_t, c, \cdot),
	       \langle \tau_1, \tau_2 \rangle)$, c.f., \cite{DvorakPichlerWoltran12}.}
	 \label{fig:stabl}%\algorithmfootnote{
	%\renewcommand{\eqdef}{{\ensuremath{\,\mathrel{\mathop:}=}}}
	 % $\MAII{S}{S'} \eqdef S \cup S'$,
	 % $\MAI{S}{e} \eqdef S \cup \{e\}$, and
	 % $\MAR{S}{e} \eqdef S \setminus \{e\}$.}
	\end{algorithm}%
	\renewcommand{\eqdef}{{\ensuremath{\,\mathrel{\mathop:}=}}}

}

%DP for COMPLETE extensions
{
	\renewcommand{\eqdef}{\leftarrow}
	%
	% NEW VERSION
	%
	%\vspace{-5ex}
	 \begin{algorithm}[h]
	   \KwData{Node~$t$, bag $\chi_t$, bag-framework~$F_t=(A_t, R_t)$,
	   credulous argument~$c$, and
	     $\langle \tau_{1}, \tau_2 \rangle$ is the sequence of tables
	     of~children of~$t$. \shortversion{
	
	     }\textbf{Out: } Table~$\tau_t.$}
	     %
	     %\KwResult{ \ADM-table~$\tau_{t}.$}
	   %
	   \lIf(\hspace{-1em})
	   {$\type(t) = \leaf$}{%
	     $\tau_{t} \eqdef \{ \langle
	     \tuplecolor{\inputPredColor}{\emptyset}, \tuplecolor{\statePredColor}{\emptyset}, \tuplecolor{\statePredColor}{\emptyset}, \tuplecolor{\outputPredColor}{\emptyset}, \tuplecolor{\outputPredColor}{\emptyset}
	     \rangle \}$%\label{line:primleaf}%
	     %
		%\EndIIf
	   }%
	  \uElseIf{$\type(t) = \intr$ and $a\hspace{-0.1em}\in\hspace{-0.1em}\chi_t$ is the introduced argum.}{
	   %\vspace{-0.05em}
	   %\makebox[3.9cm][l]{\hspace{-1em}
	   $\hspace{-1em}\tau_{t} \eqdef %\compr(
	   \{ \langle \tuplecolor{\inputPredColor}{J}, %\tuplecolor{\outputPredColor}%{\MAII{\mathcal{A}}{A_t \rightarrowtail_{R_t} J}},
	    \tuplecolor{\statePredColor}{\MAII{\mathcal{D}}{D \leftarrowtail_{R_t} J}},
	     \tuplecolor{\statePredColor}{D},
	     \tuplecolor{\outputPredColor}{\MAII{\mathcal{O}}{O \leftarrowtail_{R_t} O}},
	     \tuplecolor{\outputPredColor}{O}
	     \rangle %}
	     \mid\langle \tuplecolor{\inputPredColor}{I}, %\tuplecolor{\outputPredColor}{\mathcal{A}},
\tuplecolor{\statePredColor}{\mathcal{D}},
\tuplecolor{\statePredColor}{\mathcal{DC}},
\tuplecolor{\outputPredColor}{\mathcal{O}},
\tuplecolor{\outputPredColor}{\mathcal{OC}} \rangle\in \tau_{1}, J \in \{I, \MAI{I}{a}\},$ $D \in \{\mathcal{DC}, \MAI{\mathcal{DC}}{a}\},\shortversion{\hspace{-5em}}$ $O \in \{\mathcal{OC}, \MAI{\mathcal{OC}}{a}\},$ $J \cap D \cap O = \emptyset, $
	    % \makebox[2.9cm][l]{}
	    $J \rightarrowtail_{R_t} J = \emptyset,$ $J \rightarrowtail_{R_t} O = \emptyset,$ $O \rightarrowtail_{R_t} J = \emptyset, J \cap \{c\} = \chi(t) \cap \{c\} \}$
	      %$%)
	      %\hspace{-5em}$
	      %\label{line:primintr}
	     %
	   \vspace{-0.05em}
	     }\vspace{-0.05em}%
	     \uElseIf{$\type(t) = \rem$ \KwAnd $a \not\in \chi_t$ is the removed argum.}{%
%	       \makebox[3.9cm][l]{%\hspace{-1em}
	       $\hspace{-1em}\tau_{t} \eqdef %\compr(
	       \{ \langle \tuplecolor{\inputPredColor}{\MAR{I}{a}}, %\tuplecolor{\outputPredColor}{\MAR{\mathcal{A}}{a}},
	        \tuplecolor{\statePredColor}{\MAR{\mathcal{D}}{a}}, \tuplecolor{\statePredColor}{\MAR{\mathcal{DC}}{a}}, \tuplecolor{\outputPredColor}{\MAR{\mathcal{O}}{a}}, \tuplecolor{\outputPredColor}{\MAR{\mathcal{OC}}{a}}
	       \rangle %}
	       \mid\langle \tuplecolor{\inputPredColor}{I},
	       \tuplecolor{\statePredColor}{\mathcal{D}},
	       \tuplecolor{\statePredColor}{\mathcal{DC}},
	       \tuplecolor{\outputPredColor}{\mathcal{O}},
	       \tuplecolor{\outputPredColor}{\mathcal{OC}}
	       \rangle \in \tau_{1},$
	       $a \in I \cup \mathcal{D} \cup \mathcal{O} \}%)
	       \hspace{-5em}$%\label{line:primrem}
	       \vspace{-0.1em}
	     } %
	     \ElseIf{$\type(t) = \join$}{%, and $\Tab{} = \langle \tau', \tau'' \rangle$}{%
	       %\makebox[2.9cm][l]{\hspace{-1em}
	       $\hspace{-1em}\tau_{t} \eqdef %\compr(
	       \{ \langle \tuplecolor{\inputPredColor}{I}, %\tuplecolor{\outputPredColor}{\MAII{{\mathcal{A}_1}}{{\mathcal{A}_2}}},
	        \tuplecolor{\statePredColor}{\MAII{{\mathcal{D}_1}}{{\mathcal{D}_2}}},
	        \tuplecolor{\statePredColor}{{{\mathcal{DC}}}},
	        \tuplecolor{\outputPredColor}{\MAII{{\mathcal{O}_1}}{{\mathcal{O}_2}}},
	        \tuplecolor{\outputPredColor}{{{\mathcal{OC}}}}
		 \rangle %}
		 \mid\langle \tuplecolor{\inputPredColor}{I}, %\tuplecolor{\outputPredColor}{\mathcal{A}_1},
		 \tuplecolor{\statePredColor}{\mathcal{D}_1}, \tuplecolor{\statePredColor}{\mathcal{DC}}, \tuplecolor{\outputPredColor}{\mathcal{O}_1}, \tuplecolor{\outputPredColor}{\mathcal{OC}} \rangle\in \tau_{1},$
		 $\langle \tuplecolor{\inputPredColor}{I},
		 %\tuplecolor{\outputPredColor}{{\mathcal{A}_2}},
		 \tuplecolor{\statePredColor}{{\mathcal{D}_2}},
		 \tuplecolor{\statePredColor}{\mathcal{DC}}, \tuplecolor{\outputPredColor}{\mathcal{O}_2}, \tuplecolor{\outputPredColor}{\mathcal{OC}}\rangle\in \tau_{2}\}\hspace{-5em}$%\label{line:primjoin}
	       \vspace{-0.1em}
	     }
	     \Return $\tau_{t}$
	     \vspace{-0.25em}
	     \caption{Local algorthm~$\COMP(t, \chi_t, \cdot, (F_t, c, \cdot),
	       \langle \tau_1, \tau_2 \rangle)$, c.f., \cite{DvorakPichlerWoltran12}.}
	 \label{fig:compl}%\algorithmfootnote{
	%\renewcommand{\eqdef}{{\ensuremath{\,\mathrel{\mathop:}=}}}
	 % $\MAII{S}{S'} \eqdef S \cup S'$,
	 % $\MAI{S}{e} \eqdef S \cup \{e\}$, and
	 % $\MAR{S}{e} \eqdef S \setminus \{e\}$.}
	\end{algorithm}%
	\renewcommand{\eqdef}{{\ensuremath{\,\mathrel{\mathop:}=}}}

}

\subsection{Further Proof Details}

\begin{restateproposition}[thm:runtime:basic]
\begin{proposition}
  Algorithm~$\dpa_{\STAB}$ runs in
  time~$\bigO{3^{k} \cdot k \cdot g}$, $\dpa_{\ADM}$ in
  $\bigO{4^{k} \cdot k \cdot g}$, and~$\dpa_{\COMP}$ in
  $\bigO{5^{k} \cdot k \cdot g}$ where $k$ is the width and $g$ the
  number of nodes of the TD.
\end{proposition}
\end{restateproposition}
\begin{proof}[Proof (Sketch)]
  Let~$d = k+1$ be maximum bag size of the
  TD~$\TTT$. %For each node~$t$ of $T$, we consider the
  %table $\nu(t)$ which has been computed by
  %$\dpa_\AlgA$. %~\cite{SamerSzeider10b}.
  We only discuss the case for algorithm~$\dpa_{\ADM}$ here.
  The table~$\tau(t)$ has at most
  $4^{d}$ rows of the form~$\langle I, \mathcal{A}, \mathcal{D}\rangle$,
  since an argument actually can be either in one of these sets~$I,\mathcal{A},\mathcal{D}$
  or in none of them (just modify \ADM such that~$\mathcal{A}\cap\mathcal{D}=\emptyset$).
  In total, with the help of efficient data structures, e.g., for nodes~$t$ with~$\type(t)=\join$, one can establish a runtime bound of~$\bigO{{4^{d}}}$.
  %since we also need to store the counters in binary representation.
  %
  Then, we check within the bag for admissibility, keeping in mind only the changes and apply this to every node~$t$ of the TD, which resulting in
  running
  time~$\bigO{{4^{d}} \cdot d \cdot g}\subseteq \bigO{4^{k} \cdot k\cdot g}$.
  %Consequently, the theorem holds.
  %
\end{proof}

\begin{restateproposition}[thm:runtime]
\begin{proposition}[\citey{FichteHecher18b}]%[$\star$]
  $\dpa_{\PROJ}$ runs in time
  $\bigO{2^{4m}\cdot g \cdot \gamma(\CCard{F})}$ where $g$ is the number of nodes of the given TD of the underlying
  graph~$G_F$ of the considered AF~$F$ and
  $m\eqdef \max\{\Card{\nu(t)} \mid t\in N\}$ for input
  TTD~$\TTT_{\text{purged}} = (T,\chi,\nu)$ of $\dpa_\PROJ$.
\end{proposition}%
\end{restateproposition}
\begin{proof}
  For each
  node~$t$ of $T$, we consider the table $\nu(t)$ of $\TTT_{\text{purged}}$. %Note that~$m$ is the maximal table size depending on~$\mathbb{S}$ and the treewidth of~$G_F$.
  % containing rows
  % computed by $\dpa_\AlgA$. %~\cite{SamerSzeider10b}.
  %      %      The table~$\nu(t)$ has at most
  %      %      $2^{d}\cdot 2^{d} \cdot d!$ rows.
  %
  Let TDD~$(T,\chi,\pi)$ be the output of~$\dpa_\PROJ$. In the worst case,
  we store in~$\pi(t)$ each subset~$\rho \subseteq \nu(t)$ together
  with exactly one counter. Hence, we have at most $2^{m}$ many rows
  in $\rho$.
  %
  %Here we assume arithmetic operations have constant runtime.
  %
  In order to compute $\ipmc$ for~$\rho$, we consider every
  subset~$\varphi \subseteq \rho$ and compute~$\pmc$. Since
  $\Card{\rho}\leqslant m$, we have at most~$2^{m}$ many subsets $\varphi$
  of $\rho$. Finally, for computing $\pmc$, we consider in the worst
  case each subset of the origins of~$\varphi$ for each child table,
  which are at most~$2^{m}\cdot 2^{m}$ because of nodes~$t$
  with~$\type(t)=\join$.
  %
  % Note that many values actually overlap,
  % running dynamic programing for computing~$\pcnt$ would actually
  % improve the runtime slightly, but will complicate analysis a lot.
  %
  In total, we obtain a runtime bound
  of~$\bigO{2^{m} \cdot 2^{m} \cdot 2^{m}\cdot 2^{m} \cdot
    \gamma(\CCard{F})} \subseteq \bigO{2^{4m} \cdot
    \gamma(\CCard{F}})$ due to multiplication of two $n$-bit
  integers for nodes~$t$ with~$\type(t)=\join$ at costs~$\gamma(n)$.
  %
  % store the counters in binary representation.
  %
  Then, we apply this to every node of~$T$ %of the TD
  resulting in
  runtime~$\bigO{2^{4m} \cdot g \cdot \gamma(\CCard{F})}$.
  %Consequently, the theorem holds.
  %
\end{proof}

\subsubsection{Classical Counting Complexity}
\begin{restateproposition}[lem:numberext-in-numberp]
\begin{lemma}
$\numberCred_\SEM$ is in
\begin{enumerate}
	\item $\numberDotP$ if $\SEM$ is conflict-free, stable, admissible, or complete.
	\item $\numberDotCoNP$ if $\SEM$ is preferred, semi-stable, or stage.
\end{enumerate}
\end{lemma}
\end{restateproposition}

\begin{proof}
	The nondeterministic machine first guesses a candidate extension set $S$ and then verify whether it is an extension of the desired semantics plus if the given argument is contained in it.
	The number of computation paths then one-to-one corresponds to possible extensions.
\begin{enumerate}
\item Being conflict-free can be checked in $\Ptime$.
  \citex{DBLP:conf/ecsqaru/Coste-MarquisDM05} show that the
  verification process of extensions for the semantics admissible, stable,
  and complete can be done in deterministic polynomial time.
	
	\item For semi-stable, resp., stage extensions, we need to ensure that there exists no set $S'\subseteq A$ whose range is a superset of the range of the extension candidate.
		This property can be verified with a $\co\NP$ oracle.
		Similarly, \citex{DunneBench-Capon02a} claim that verifying if a given extension is preferred is $\co\NP$-complete.\qedhere%
\end{enumerate}%
\end{proof}

\begin{restateproposition}[lem:numberext-numberp-hard]
\begin{lemma}
$\numberCred_\SEM$ is
\begin{enumerate}
	\item $\numberDotP$-hard under $\parsi$ reductions if $\SEM$ is stable, admissible, or complete.
	\item $\numberDotCoNP$-hard under $\subtr$ reductions if $\SEM$ is semi-stable, or stage.
\end{enumerate}
	
\end{lemma}
\end{restateproposition}
\begin{proof}[Proof (Sketch)]
1. 
Start with the case of stable or complete extensions. 
Adopting ideas of \citex{DunneBench-Capon02a}, we construction a parsimonious reduction from $\#\SAT$.
Given a propositional formula $\varphi(x_1,\dots,x_n)=\bigwedge_{i=1}^m C_i$ with clauses $C_i$, define an AF $F_\varphi=(A,R)$ where
	\begin{align*}
	A &= \SB x_i,\bar x_i\SM 1\leqslant i\leqslant n\SE\cup\SB C_i\SM 1\leqslant i\leqslant m\SE\cup \{t,\bar t\},\\
	R &= \SB (x_i,\bar x_i),(\bar x_i,x_i)\SM 1\leqslant i\leqslant n\SE\\
	&\,\cup\SB (x_i,C_j)\SM x_i\in C_j\SE\cup\SB(\bar x_i,C_j)\SM\bar x_i\in C_j\SE\\
	&\,\cup\SB (C_i,t)\SM 1\leqslant i\leqslant m\SE\cup\{(t,\bar t),(\bar t,t)\}.
	\end{align*}
        Then, due to the range maximality, the number of satisfying
        assignments of $\varphi$ coincides with the number of stable
        (complete) extensions of $F_\varphi$ which contain the
        argument $t$.
	
        For the case of admissible extensions, to count correctly, it is crucial that for
        each variable $x_i$ either argument $x_i$ or $\bar x_i$ is
        part of the extension.  To ensure this, we introduce arguments
        $s_1,\dots,s_n$ attacking $t$ that can only be defended by one
        of $x_i$ or $\bar x_i$.  As a result, for each admissible
        extension $S$, we have that $\Card{S\cap\{x_i,\bar x_i\}}=1$
        for each $1\leqslant i\leqslant n$.  The modified framework
        for this case then is $F_\varphi'=(A',R')$, where
	\begin{align*}
		A' &= A\cup \SB s_i\SM 1\leqslant i\leqslant n\SE,\\
		R' &= R\cup \SB (s_i,t),(x_i,s_i),(\bar x_i,s_i)\SM 1\leqslant i\leqslant n\SE.
	\end{align*}

2. We state a parsimonious reduction from counting minimal models of CNFs to the $\numberCred_\SEM$ problem.
The formalism of circumscription is well-established in the area of AI \cite{McCarthy80}. 
Formally, one considers assignments of Boolean formulas that are \emph{minimal} regarding the \emph{pointwise partial order} on truth assignments:
if $s=(s_1,\dots,s_n),s'=(s_1',\dots,s_n')\in\{0,1\}^n$, then write $s<s'$ if $s\neq s'$ and $s_i\leqslant s_i'$ for every $i\leqslant n$.
% Then, we define the following counting problem:
%
% \countproblem{$\countCirc$}{A Boolean formula $\varphi$ in CNF.}{Number of minimal models of $\varphi$.}
%
Then, we define the problem $\countCirc$ which asks given a Boolean
formula $\varphi$ in CNF to output the number of minimal models of
$\varphi$.
\citex{DurandHermannKolaitis05} showed that $\countCirc$ is
$\numberDotCoNP$-complete via subtractive reductions.  Given a Boolean
formula $\varphi(x_1,\dots,x_n)=\bigwedge_{i=1}^mC_i$ with $C_i$ are
disjunctions of literals, we will construction an argumentation
framework $F_\varphi=(A,R)$ as follows:
\begin{align*}
	A &= \SB x_i,\bar x_i, b_i\SM 1\leqslant i\leqslant n\SE\cup\SB C_i\SM 1\leqslant i\leqslant m\SE\cup \{t\},\\
	R &= \SB (b_i,b_i),(\bar x_i,b_i),(x_i,\bar x_i),(\bar x_i,x_i)\SM 1\leqslant i\leqslant n\SE\\
	&\,\cup\SB (x_i,C_j)\SM x_i\in C_j\SE\cup\SB(\bar x_i,C_j)\SM\bar x_i\in C_j\SE\\
	&\,\cup\SB (C_i,t)\SM 1\leqslant i\leqslant m\SE.
\end{align*}

The crux is that choosing negative literals is more valuable than selecting positive ones.
This is true as each negative literal additionally attack a corresponding $b_i$ and thereby increases the range (more than the positive literal could). 
Consequently, this construction models subset minimal models.
Finally, one merely needs to select models where $t$ is in a range-maximal semi-stable, resp., stage extension.
\end{proof}

\begin{restateproposition}[lem:proj-upperbound]
\begin{lemma}
	$\PCC_\SEM$ is in
	\begin{enumerate}
		\item $\#\cdot\NP$ if $\SEM$ is stable, admissible, or complete.
		\item $\#\cdot\Sigma_2^\Ptime$ if $\SEM$ is semi-stable, or stage.
	\end{enumerate}
\end{lemma}
\end{restateproposition}

\begin{proof}[Proof (Sketch)]
	Given an argumentation framework $AF$, a projection set $P$, and an argument $a$.
	Nondeterministically branch on a possible projected extension $S$.
	Accordingly we have $S\subseteq P$.
	If $a\in S$ and $S$ is of the respective semantics, then accept.
	Otherwise make the one allowed nondeterministic oracle guess $S'\supseteq S$, verify if $P\cap S'=S$, $a\in S'$, and $S'$ is of the desired semantics.
	As explained in the proof of Lemma~\ref{lem:numberext-in-numberp} extension verification is (1.) in $\Ptime$ for stable, admissible, or complete, and (2.) in $\co\NP$ for semi-stable, or stage.
	Concluding, we get an $\NP$ oracle call for the first case, and an $\NP^{\co\NP}=\NP^\NP=\Sigma_2^\Ptime$ oracle call in the second case.
	This yields either $\#\cdot\NP$ or $\#\cdot\Sigma_2^\Ptime$ as upper bounds.
\end{proof}

\begin{restateproposition}[lem:pcct-numberdotcnp-hard]
\begin{lemma}
	$\PCC_\SEM$ is
	\begin{enumerate}
		\item $\#\cdot\Sigma_2^\Ptime$-hard w.r.t.\ $\parsi$ reductions if $\SEM$ is stage, or semi-stable.
		\item $\#\cdot\NP$-hard w.r.t.\ $\parsi$ reductions if $\SEM$ is admissible, stable, or complete.
	\end{enumerate}
\end{lemma}
\end{restateproposition}
\begin{proof}[Proof (Sketch)]
1. We state a parsimonious reduction from $\#\Sigma_2\SAT$ to $\PCC_\SEM$.
We use an extended version of the construction of \citex{DvorakWoltran10}.
Given a formula $\varphi(X)=\exists Y\forall Z\;\psi(X,Y,Z)$, where $X,Y,Z$ are sets of variables, and $\psi$ is a DNF.
Consider now the negation of $\varphi(X)$, i.e., $\varphi'(X)=\lnot\varphi(X)\equiv\forall Y\exists Z\; \lnot\psi(X,Y,Z)$.
Let $\psi'(X,Y,Z)$ be $\lnot\psi(X,Y,Z)$ in NNF.
Accordingly, $\psi'$ is a CNF, $\psi'(X,Y,Z)= \bigwedge_{i=1}^m C_i$ and $C_i$ is a disjunction of literals for $1\leqslant i\leqslant m$.
Note that, the formula $\varphi'(X)$ is of the same kind as the formula in the construction of \citex{DvorakWoltran10}.
Now define an argumentation framework $AF=(A,R)$, where
	\begin{align*}
		A &= \SB x, \bar x\SM x\in X\SE\cup \SB y, \bar y, y', \bar y'\SM y\in Y \SE\\
		  &\;\cup\SB z,\bar z\SM z\in Z\SE\cup\{t,\bar t, b\}\\
		R &= \SB(y',y'), (\bar y',\bar y'), (y,y'), (\bar y,\bar y'), (y,\bar y), (\bar y,y)\SM y\in Y \SE\\
		&\;\cup\{(b,b),(t,\bar t),(\bar t,t), (t,b)\}\\
		&\;\cup\SB(C_i,t) \SM 1\leqslant i\leqslant m \SE\\
		&\;\cup\SB(u,C_i) \SM u\in X\cup Y\cup Z, u\in C_i, 1\leqslant i\leqslant m \SE\\
		&\;\cup\SB(\bar u,C_i) \SM z\in X\cup Y\cup Z, \bar u\in C_i, 1\leqslant i\leqslant m \SE
	\end{align*}
Note that, by construction, the $y',\bar y'$ variables make the extensions w.r.t.\ the universally quantified variables $y$ incomparable.
Further observe that choosing $t$ is superior to selecting $\bar t$, as $t$ increases the range by one more.
(This is crucial in our case, as stage as well as semi-stable strive for range maximal extensions.)
If for every assignment over the $Y$-variables there exists an assignment to the $Z$-variables, then, each time, when there is a possible solution to $\psi'(X,Y,Z)$, so semantically $\neg\psi(X,Y,Z)$, w.r.t.\ the free $X$-variables, the extension will contain $t$.
As a result, the extensions containing $t$ correspond to the unsatisfying assignments.
Let $A(\varphi(X))$ be the set of assignments of a given $\#\Sigma_2\SAT$-formula, and $B(AF,P,a)$ be the set of stage/semi-stable extensions which contain $a$ and are projected to $P$.
Then, one can show that $\Card{A(\varphi(X))}=\Card{B(AF,X,\bar t)}$ proving the desired reduction (as $\bar t$ together with the negation of $\varphi(X)$ in the beginning, intuitively, is a double negation yielding a reduction from $\#\Sigma_2\SAT$).

2. Now turn to the case of admissible, stable, or complete extensions.
Again, we provide a similar parsimonious reduction, but this time,  from $\#\Sigma_1\SAT$ to $\PCC_\SEM$.
Consider a formula $\varphi(X)=\exists Y\; \psi(X,Y)$, where $X,Y$ are sets of variables, $\psi = \bigwedge_{i=1}^m C_i$ and $C_i$ is a disjunction of literals for $1\leqslant i\leqslant m$.
Essentially the reduction is the same, however we need the same  extension as in the proof of Lemma~\ref{lem:numberext-numberp-hard} and we neither need the $y',\bar y'$ nor---of course---the $z$ variables.
Define the framework $AF=(A,R)$ as follows:
\begin{align*}
	A &= \SB x, \bar x\SM x\in X\SE\cup \SB y, \bar y\SM y\in Y \SE\\
	  &\;\cup\{t,\bar t, b\}\cup\SB s_x,s_y\SM x\in X, y\in Y\SE\\
	R &= \{(b,b),(t,\bar t),(\bar t,t), (t,b)\}\\
	&\;\cup\SB(C_i,t) \SM 1\leqslant i\leqslant m \SE\\
	&\;\cup\SB(s_x,t),(s_y,t)\SM x\in X, y\in Y\SE\\
	&\;\cup\SB(x,s_x),(\bar x,s_x),(y,s_y),(\bar y,s_y)\SM x\in X, y\in Y\SE\\
	&\;\cup\SB(u,C_i) \SM u\in X\cup Y, u\in C_i, 1\leqslant i\leqslant m \SE\\
	&\;\cup\SB(\bar u,C_i) \SM z\in X\cup Y, \bar u\in C_i, 1\leqslant i\leqslant m \SE
\end{align*}
This time, let $A(\varphi(X))$ denote the set of satisfying assignments of an $\Sigma_1\SAT$ instance.
Then, define $B(AF,P,a)$ be the set of admissible/stable/complete extensions which contain $a$ and are projected to $P$.
Finally, one can show that $\Card{A(\varphi(X))}=\Card{B(AF,X,t)}$ showing the claimed reduction and $\#\cdot\NP$-hardness via parsimonious reductions.
\end{proof}%

%END longversion
%
}%

\end{document}